\documentclass[opre,nonblindrev]{informs3aa} 
\OneAndAHalfSpacedXI 

\usepackage{endnotes}
\let\footnote=\endnote
\let\enotesize=\normalsize
\def\notesname{Endnotes}%
\def\makeenmark{$^{\theenmark}$}
\def\enoteformat{\rightskip0pt\leftskip0pt\parindent=1.75em
  \leavevmode\llap{\theenmark.\enskip}}
\usepackage{natbib}
 \bibpunct[, ]{(}{)}{,}{a}{}{,}%
 \def\bibfont{\small}%
 \def\bibsep{\smallskipamount}%
 \def\bibhang{24pt}%
\usepackage[colorlinks=true,breaklinks=true,bookmarks=false,urlcolor=blue,
     citecolor=blue,linkcolor=blue,bookmarksopen=false,draft=false]{hyperref}

\def\EMAIL#1{\href{mailto:#1}{#1}}

\usepackage{booktabs,multirow,graphics,latexsym,amsfonts,epsfig,verbatim,amsmath,algorithm,algorithmic,color,tabularx,mathtools,subcaption}
\def\QED{\hfill \quad{\bf Q.E.D.}\medskip}

\TheoremsNumberedThrough     
\ECRepeatTheorems
\EquationsNumberedThrough    

\begin{document}

\RUNAUTHOR{Deng and Jia}
\RUNTITLE{Decentralized Contextual Bandits with Network Adaptivity}
\TITLE{Decentralized Contextual Bandits with Network Adaptivity}

\ARTICLEAUTHORS{%
\AUTHOR{Chuyun Deng, Huiwen Jia}
\AFF{Department of Industrial Engineering and Operations Research, University of California, Berkeley\\ 
\{\EMAIL{chuyun\_deng@berkeley.edu, huiwenj@berkeley.edu}\}} 
} 

\ABSTRACT{
We consider contextual linear bandits over networks, a class of sequential decision-making problems where learning occurs simultaneously across multiple locations and the reward distributions share structural similarities while also exhibiting local differences.
While classical contextual bandits assume either fully centralized data or entirely isolated learners, much remains unexplored in networked environments when information is partially shared.
In this paper, we address this gap by developing two network-aware Upper Confidence Bound (UCB) algorithms, NetLinUCB and Net-SGD-UCB, which enable adaptive information sharing guided by dynamically updated network weights. Our approach decompose learning into global and local components and as a result allow agents to benefit from shared structure without full synchronization.
Both algorithms incur lighter communication costs compared to a fully centralized setting as agents only share computed summaries regarding the homogeneous features.
We establish regret bounds showing that our methods reduce the learning complexity associated with the shared structure from $O(N)$ to sublinear $O(\sqrt{N})$, where $N$ is the size of the network. 
The two algorithms reveal complementary strengths: NetLinUCB excels in low-noise regimes with fine-grained heterogeneity, while Net-SGD-UCB is robust to high-dimensional, high-variance contexts. 
We further demonstrate the effectiveness of our methods across simulated pricing environments compared to standard benchmarks.
}


\maketitle

\section{Introduction}
Dynamic decision-making over networked systems is central to many real-world applications, including city-level dynamic pricing \citep[e.g.,][]{qu2025decision}, distributed recommendation systems \citep[e.g.,][]{cai2024distributed}, and decentralized resource allocation \citep[e.g.,][]{asadpour2020online}. In this paper, we focus on a contextual multi-armed bandit (MAB) framework motivated by a dynamic pricing problem faced by online on-demand platforms such as Uber and DoorDash. Consider a scenario in which a platform must set prices across multiple locations for the same service  (e.g., a ride or food delivery). Each location may exhibit distinct demand patterns. For example, a tourist-heavy city may experience seasonal surges during holidays or festivals, while another may be more sensitive to commuter traffic or local supply constraints. At the same time, cities often share structural factors, such as population density, competitor presence, or event schedules, which influence demand in similar ways. To optimize revenue, the platform needs to learn effective pricing strategies over time by leveraging both local signals and transferable knowledge across cities. This setting naturally lends itself to a contextual multi-armed bandit (MAB) formulation, which has emerged as a core framework for sequential decision-making in personalized and adaptive environments \citep{abbasi-yadkori_improved_2011,li_contextual-bandit_2010,chu_contextual_2011,martinez-rubio_decentralized_2019,mahadik_fast_2020,jia_online_2022}.

A straightforward application of contextual MAB to this decision-making problem yields two benchmark approaches. Both formulate each discrete price level as an arm and treat the observed context, including both shared and city-specific features that affect demand, as a whole context. The first constructs an independent contextual MAB for each city, learning pricing strategies based solely on local observations. Although this respects local heterogeneity, it fails to leverage shared structure across cities. The second approach aggregates all data into a single centralized contextual MAB, maximizing information sharing but at the cost of increased communication and a loss of granularity in capturing local demand nuances. We further introduce the details of these two benchmark algorithms in Sections \ref{sec:benchmark}.

Our goal is to develop algorithms that balance the extremes of fully decentralized and fully centralized learning, enabling effective knowledge transfer across the network while preserving node-level autonomy and minimizing communication overhead. This introduces a new layer of complexity. At the inner layer, we face the classical exploration–exploitation dilemma inherent to online learning, typically measured by cumulative regret. At the outer layer, we must manage a distinct trade-off between communication cost and learning performance: greater communication can accelerate learning and reduce regret, but at the expense of efficiency and scalability. Designing algorithms that capture both levels of trade-off is central to our work.

Recent work has explored various directions to bridge this gap \citep[e.g.,][]{martinez-rubio_decentralized_2019,kolla_collaborative_2016}.
However, these approaches either rely on a central server or focus on a non-contextual setting. To address these limitations, we propose a new class of decentralized contextual bandit algorithms that enable principled information sharing across agents through dynamically learned network weights, while preserving local autonomy and operating without a central server. Our approach builds upon and extends previous work by incorporating real-time contextual updates, decentralized gradient-based confidence bounds, and variance-sensitive exploration strategies.

\subsection{Related Work}

We refrain from surveying the extensive bandit literature and instead focus on three relevant clusters of work.

\textbf{Upper Confidence Bound Methods.} To address the core exploration-exploitation trade-off, two main approaches prevail, UCB-based methods~\citep{chu_contextual_2011,abbasi-yadkori_improved_2011,garivier_upper-confidence_2011,wang_multi-modal_2019,zhou_neural_2020}, which construct optimistic confidence bounds for efficient exploration, and Thompson Sampling~\citep{agrawal_further_2012,agrawal_thompson_2013,ferreira_online_2018,neu_lifting_2022,xu_noise-adaptive_2023}, which leverages posterior sampling for probabilistic action selection. Beyond these, a growing body of alternative methods has emerged, including information-theoretic~\citep{zhou_learning_2019,marinov_pareto_2021}, PAC-based~\citep{li_instance-optimal_2022}, and neural strategies~\citep{sezener_online_2020,zhou_neural_2020,qi_robust_2024}. These methods have been successfully applied to real-world applications and deployed in industrial-scale systems \citep{li_contextual-bandit_2010,li_exploitation_2010,durand_contextual_2018,jia_online_2022-1,jia2024online}. In this paper, we propose two UCB-based algorithms achieving sublinear regret on the learning horizon and the network size.

\textbf{Interactive Bandits.} Recent advances in interactive bandits provide promising directions to address this challenge and opportunity arising from richer structure beyond classic bandit settings, including clustering bandits~\citep{korda_distributed_2016,ghosh_multi-agent_2022}, multi-task linear bandits~\citep{yang_impact_2021,hu_near-optimal_2021}, and bandits with graph feedback \citep{wen_stochastic_2024}.
Unlike our setting where we consider a network, i.e., graph, over contextual bandit problems, \citet{wen_stochastic_2024} considers a graph of arms and pulling one arm results in several reward feedback. \citet{szorenyi_gossip-based_2013} combined peer-to-peer communication networks with the $\epsilon$‑greedy strategy for distributed multi‑armed bandits and \citet{martinez-rubio_decentralized_2019} proposed a fully decentralized UCB model to address the classic MAB problem. We adopt the same problem formulation but extend it to the contextual setting. 
\citet{wang_distributed_2020} studied communication‑efficient distributed linear bandits where multiple agents interact with a central server by periodically exchanging information, under the assumption that all agents share the same bandit model. Building on this setup, \citet{do_multi-agent_2023} introduced an $\epsilon$‑similarity constraint on model parameters to allow for agent-level personalization, while still relying on centralized communication. Similarly, \citet{dubey_differentially-private_2020} enhanced privacy through differentially private updates and periodic aggregation, which also retains a centralized architecture. In contrast, our work avoids reliance on a central server and focuses on decentralized real-time updates. 

\textbf{Global and Local Learning.} A parallel line of research explores the decomposition of global and local learning. \citet{shi_federated_2021} introduced a combination of global and local exploration with multiple simultaneous pulls and employed trimming techniques to preserve personalization. \citet{huang_federated_2021} developed federated linear bandits that share global parameters while retaining local ones, but their updates are not performed in real time. \citet{li_asynchronous_2021} followed a similar pattern by jointly estimating global parameters for online learning, achieving an \(O(\sqrt{T} \log T)\) regret bound without explicitly designing weights. While their approach relies on global parameter updates, it does not incorporate fine-grained similarity across agents. Meanwhile, \citet{kolla_collaborative_2016} studied networked multi-armed bandits, where each node exploits neighbors’ arm-selection patterns to improve decisions. Inspired by this idea, we extend neighborhood collaboration to the contextual linear setting using a fully connected graph, where edge weights reflect both arm-selection frequency and context similarity. Our algorithms leverage this structure to perform real-time updates that balance global sharing with agent-specific learning.

Lastly, we would like to mention two works closely related to our second algorithm, Net-SGD-UCB. \citet{ba_doubly_2024} explored doubly optimal learning in monotone games, highlighting the role of gradient sharing in distributed settings. These models strongly motivate our Net-SGD-UCB algorithm, which unifies confidence-based exploration with decentralized gradient aggregation. In addition, \citet{zhou_neural_2020} inspires our design of a novel gradient-based confidence radius, transitioning from the paper's form  \(g^\top G g\) to a context-aware variant \(\mathbf{x}^\top G \mathbf{x}\), allowing gradient information to be integrated more directly with contextual structure and supporting decentralized, variance-sensitive exploration.

\subsection{Our Contributions}

To build a network-aware bandit framework, we develop two online algorithms, NetLinUCB and Net-SGD-UCB, that extend linear contextual bandits to decentralized, multi-agent settings with adaptive information sharing. 
\begin{itemize}
    \item NetLinUCB builds on the classical LinUCB framework~\citep{li_contextual-bandit_2010,chu_contextual_2011} and introduces a neighbor-weighted estimator. Each agent maintains its own local parameters but sequentially incorporates estimates from its neighbors to determine the point estimate and corresponding confidence radius, weighted according to exploration statistics and contextual similarity. 
    \item Net-SGD-UCB applies stochastic gradient descent (SGD) with momentum~\citep{bottou_stochastic_2012,kingma_adam_2017}, combined with adaptive learning rates inspired by AdaGrad~\citep{duchi_adaptive_2011}. It performs local updates for each agent while incorporating network aggregation to guide learning. This variant is particularly suited to high-dimensional contexts and streaming data scenarios, where memory and computation are limited. 
\end{itemize}
 
We prove that NetLinUCB achieves \(O( \sqrt{\frac{d_cT N}{c} \log (\frac{cTN}{d_c})} + \sum_{i=1}^{N} d_{i, s} \sqrt{T \log T})\) regret, where \(N\) is the number of nodes, \(d_c\) is the shared feature dimension, \(d_{i,s}\) is the node-specific feature dimension, and \( c \) is a context‑diversity constant, defined as a lower bound on the minimum eigenvalue of the empirical context covariance matrix, and it depends on the common feature dimension \(d_c\). Net-SGD-UCB achieves a regret of 
\(O( \frac{\sqrt{ N d_c T \log (\sigma^2 T})}{1 - \gamma} + \sum_{i=1}^N \frac{\sqrt{2d_{i,s}T \log (\sigma^2 T)}}{(1 - \gamma)\sigma} + \sigma N \sqrt{\frac{(1 - \mu)T\log T}{(1 + \mu)(1 - \gamma)}})\), where \(\gamma\) is the EMA smoothing factor, \(\mu\) is the momentum parameter, \(d_c\) is the shared feature dimension, \(d_{i,s}\) is the node-specific feature dimension, and \(\sigma\) is the variance of noise. We compare the proposed algorithms with two benchmark models in Section \ref{sec:insight}. In summary, our methods reduce the regret associated with the shared feature space from a linear dependence on the number of nodes $O(N)$ to sublinear $O(\sqrt{N})$, significantly improving sample efficiency.

We validate our models through extensive experiments simulating a distributed pricing system over a network of cities. Both NetLinUCB and Net-SGD-UCB consistently outperform benchmark algorithms, demonstrating faster convergence and significantly lower cumulative regret. The results also reveal several important managerial insights. First, when network connectivity increases, our models exhibit more efficient learning by leveraging information across cities. Second, NetLinUCB is especially effective when reward differences are small, leveraging fine-grained similarity between cities to guide arm selection. Third, Net-SGD-UCB demonstrates strong robustness in the presence of high-variance or noisy contextual features, benefiting from adaptive learning and variance-aware exploration.

\subsection{Structure of the Paper}
The remainder of this paper is organized as follows. In Section~\ref{sec:problem}, we introduce the decentralized contextual bandit problem and formalize the regret definition under networked decision-making. Section~\ref{sec:online-algo} presents the benchmark algorithms (Section~\ref{sec:benchmark}), Disjoint LinUCB and Shared LinUCB, and describes our proposed methods: NetLinUCB (Section~\ref{sec:netlinucb}) and Net-SGD-UCB (Section~\ref{sec:sgd-ucb}), which incorporate an adaptive network weight matrix (Section~\ref{sec:weight}). In Section~\ref{sec:regret-bound}, we provide theoretical guarantees for all algorithms, including detailed regret bounds and structural insights. Section~\ref{sec:numerical} reports numerical experiments across multiple instances, including convergence rates with respect to time steps \(T\) and number of nodes \(N\). In Section~\ref{sec:conclusion}, we conclude the paper and discuss future research directions. For clarity, all major notation is summarized in Appendix~\ref{app:notation}, and detailed algorithmic steps and proofs are provided in Appendices~\ref{app:disjoint_linucb} through~\ref{app:numerical}.

\section{Problem Formulation}\label{sec:problem}

We consider a network of \(N\) nodes. At each node, a decentralized agent is solving a linear contextual MAB instance over \(T\) synchronous rounds. The agents share a common problem structure but differ in their local model parameters to capture node-level heterogeneity. At each time step \(t\), the agent at node \(i \in [N]\) observes a context vector \(\mathbf{x}_{i,t} \in \mathbb{R}^{d_i}\), which consists of a shared component \(\mathbf{x}_{i,c,t} \in \mathbb{R}^{d_c}\) and a node-specific component \(\mathbf{x}_{i,s,t} \in \mathbb{R}^{d_{i,s}}\), such that \(\mathbf{x}_{i,t} = [\mathbf{x}_{i,c,t},\ \mathbf{x}_{i,s,t}]\) with \(d_i = d_c + d_{i,s}\). The agent then selects an arm \(a_{i,t} \in \mathcal{A} = \{a^{(1)}, \ldots, a^{(K)}\}\), receives the corresponding reward \(r^k_{i,t}\), and updates its local history \(\mathcal{D}_{i,t} = \{(\mathbf{x}_{i,\tau}, a_{i,\tau}, r_{i,\tau})\}_{\tau=1}^{t}\).

For each node \(i\) and arm \(a^{(k)}\), the reward is modeled as a linear function of the context, \(
r_{i,t}^k = \mathbf{x}_{i,t}^\top \theta_i^k + \epsilon_{i,t}^k\), where \( \epsilon_{i,t}^k \sim \mathcal{N}(0, \sigma_k^2) \) is zero-mean Gaussian noise. The expected reward is therefore 
\(\mathbb{E}[r_{i,t}^k \mid \mathbf{x}_{i,t}] = \mathbf{x}_{i,t}^\top \theta_i^k\). The linear reward structure is widely adopted in the contextual bandit literature
\citep{li_contextual-bandit_2010,abbasi-yadkori_improved_2011,dimakopoulou_estimation_2018,wang_multi-modal_2019,huang_federated_2021,xu_noise-adaptive_2023}.
Each arm \(a^{(k)}\) in node \(i\) is associated with an unknown coefficient vector, \(\mathbf{\theta}_i^k = [\mathbf{\theta}_c^k,\ \mathbf{\theta}_{i,s}^k] \in \mathbb{R}^{d_i}
\), where \( \mathbf{\theta}_c^k \in \mathbb{R}^{d_c} \) is the shared parameter across all nodes, and \( \mathbf{\theta}_{i,s}^k \in \mathbb{R}^{d_{i,s}} \) is the node-specific parameter for node \(i\). Under this decomposition, the expected reward for arm \( a^{(k)} \) in node \( i \) at time \( t \) is \(
\mathbb{E}[r_{i,t}^k \mid \mathbf{x}_{i,t}] = [\mathbf{x}_{i, c}, \mathbf{x}_{i,s}]^\top [\mathbf{\theta}_c^k, \mathbf{\theta}_{i,s}^k]
\).

\begin{remark} \textbf{Communication Network.}
In the above setting, we consider a collection of linear contextual bandit problems, where the shared feature space across all nodes has dimension at least one, i.e., $d_c \geq 1$. This condition implies that communication can be helpful between any two nodes. It is equivalent to initializing communication network as fully connected. Importantly, both our proposed algorithms dynamically construct an adaptive weight matrix $\Omega$ to guide inter-node information sharing, introduced in Section \ref{sec:weight}. As learning progresses, this mechanism allows the network structure to evolve from fully connected to partially connected or even disjoint structures—based on the informational relevance among agents.
\end{remark}

Let \(a_{i,t}^*\) denote the optimal arm, \(a_{i,t}^* = \arg\max_{k \in \mathcal{A}} \mathbf{x}_{i,t}^\top \theta_i^k\), and \(a_{i,t}\) denote the selected arm by the algorithm for node $i$ at time step $t$. The objective is to maximize cumulative expected reward across all nodes, \(
\max \mathbb{E}\left[ \sum_t \sum_i r_{i,t}^{a_{i,t}} \right]\). We adapt the notion of regret and per-trial payoff \citep[e.g.,][]{li_contextual-bandit_2010,li_unbiased_2011} as the performance metric.
The \emph{instantaneous regret} quantifies the difference between the expected reward of the optimal arm and that of the arm selected by the algorithm at a single time step. The \emph{cumulative regret} aggregates this difference across all trials, measuring the gap between the optimal cumulative reward and the actual cumulative reward.

\begin{definition}[Cumulative Regret]\label{def:regret}
The cumulative regret over the network is defined as
\[
R(T) = 
\mathbb{E}\left[\sum_{t=1}^T \sum_{i=1}^N r_{i,t}^{a_{i,t}^*} \right]
- 
\mathbb{E}\left[\sum_{t=1}^T \sum_{i=1}^N r_{i,t}^{a_{i,t}} \right],
\]
where \(a_{i,t}^*\) is the optimal arm for node \(i\) at time step \(t\), and \(a_{i,t}\) is the arm chosen by the algorithm.
\end{definition}

\begin{assumption}[Bounded contexts and parameters]\label{assump:bounded}
The context vectors \(\mathbf{x}_{i,t}\) and arm-specific parameter vectors \(\theta_i^k\) are bounded such that \(\|\mathbf{x}_{i,t}\|_2 \leq 1\), \(\|\theta_i^k\|_2 \leq 1\) for all \(i\), \(k\), and \(t\).
\end{assumption}

\begin{remark}
Assumption~\ref{assump:bounded} is standard in linear bandit analysis. It guarantees that the expected rewards are uniformly bounded, which is essential for applying concentration inequalities and deriving regret bounds. In practice, such boundedness can be achieved via normalization of both features and parameters.
\end{remark}

\section{Online Algorithms}\label{sec:online-algo}
We present four algorithms in this section, two benchmarks and two newly proposed algorithms. 
In Section \ref{sec:benchmark}, we introduce two benchmark algorithms adapted from existing contextual bandits literature and tailored for our setting. In Sections \ref{sec:netlinucb} and \ref{sec:sgd-ucb}, we propose two new algorithms specifically designed to facilitate the learning with the consideration of the shared similarities across the network. We derive the theoretical performance guarantee for each of them in Section \ref{sec:regret-bound}, and then evaluate their numerical performance in Section \ref{sec:numerical}. Both proposed algorithms consider information sharing across the network. Specifically, \textbf{NetLinUCB} enables cross-node sharing via adaptive weights, while \textbf{Net-SGD-UCB} combines momentum SGD with UCB to enable efficient decentralized learning and adaptive confidence bound construction.

\begin{figure}[htbp]
    \centering
    \includegraphics[width=0.7\columnwidth]{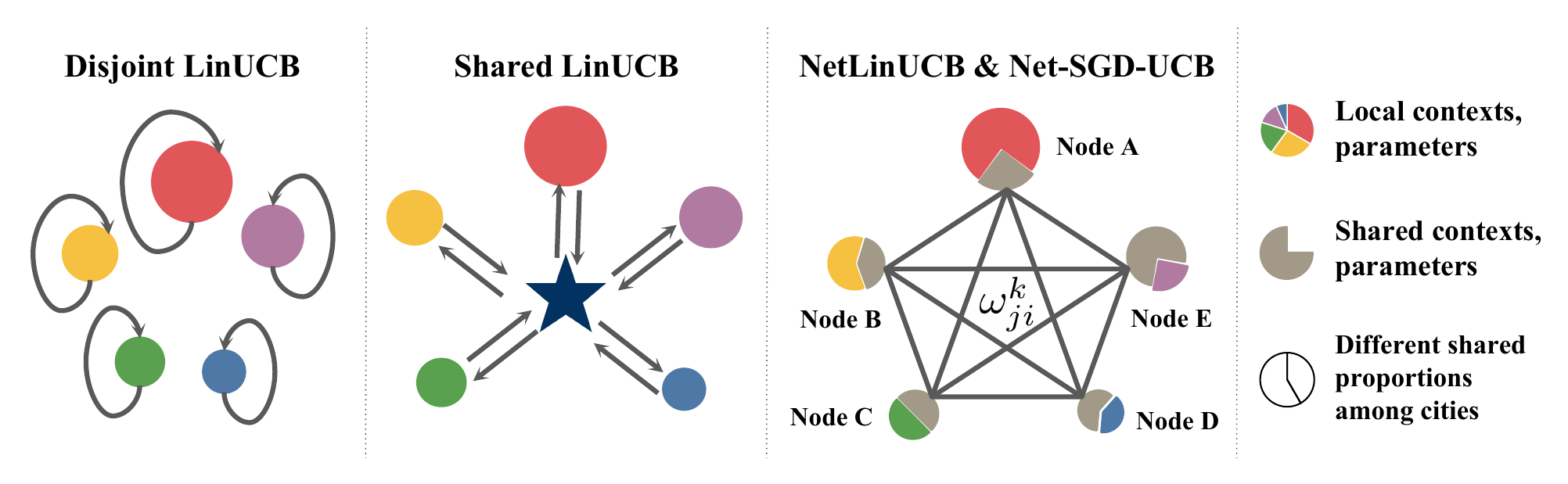}
    \caption{Information sharing structures.}
    \label{fig:model-comparison}
\end{figure}

\subsection{Benchmark Algorithms}\label{sec:benchmark}

We consider two benchmark algorithms, Disjoint LinUCB and Shared LinUCB, adapted from the well-established contextual bandit literature~\citep[][]{abbasi-yadkori_improved_2011,chu_contextual_2011}.
Disjoint LinUCB treats each node independently, ignoring any similarities across nodes. In contrast, Shared LinUCB aggregates all data and treats all nodes as a single environment with shared parameters. These two methods represent two extremes in modeling cross-node heterogeneity and serve as natural baselines to evaluate the benefits of network-aware learning.

In \textbf{Disjoint LinUCB}, each node $i \in [N]$ maintains an independent LinUCB model and estimates the parameters of all arms with its own local contexts only. At each round, the agent on each node selects the arm that maximizes the upper confidence bound, which is derived based on ridge regression and a context-aware exploration radius. The cumulative regret is \(O\left( \sum_{i=1}^N d_i \sqrt{T \log T} \right)\) as established in Theorem 3 of~\citet{das_linear_2024}. We provide detailed algorithmic steps and regret analysis in Appendix~\ref{sec:disjoint_formulation} and~\ref{sec:disjoint_analysis}.

In \textbf{Shared LinUCB}, which is a centralized variant, a central agent models all nodes within a unified parameter space by concatenating shared and node-specific features into a global context vector. For each node, the context corresponds to a distinct sub-block of the global vector, resulting in a block-sparse embedding. Shared LinUCB applies LinUCB over the aggregated context space, yielding a regret bound of \(O\left( (d_c + \sum_{i=1}^{N} d_{i,s}) \sqrt{T \log T} \right)\), extending the classical result of \citep{das_linear_2024} to a higher-dimensional parameter setting. Detailed algorithmic steps and proofs are deferred to Appendix~\ref{sec:whole_formulation} and~\ref{sec:whole_analysis}.

Disjoint LinUCB is computationally efficient but lacks information sharing, resulting in slow convergence. Shared LinUCB mitigates this by leveraging shared features, but incurs high computational cost due to its large parameter space. Our proposed algorithms, NetLinUCB and Net-SGD-UCB, address both limitations through structured information exchange and scalable, decentralized estimation. In Section~\ref{sec:numerical}, we demonstrate that our models outperform Disjoint LinUCB through cross-agent learning and achieve faster convergence without sacrificing local adaptivity. Moreover, our methods match or exceed the performance of Shared LinUCB while incurring significantly lower computational cost. These improvements highlight the benefits of network-aware estimation and structured exploration in decentralized contextual bandit settings.

\subsection{Adaptive Weight Matrix}\label{sec:weight}

Benchmark algorithms assume uniform influence across nodes, treating each node independently or equally when aggregating information. In particular, they do not adaptively assign importance to data from different nodes, which may limit performance in the presence of structural heterogeneity.

We first present an essential part of our two newly proposed algorithms, a dynamic weight matrix $\Omega^k$ defined for each arm $a^{(k)} \in \mathcal{A}$, which captures evolving similarities between nodes. Each entry $\omega_{ji}^k \in \Omega^k$ quantifies the directional influence from node $j$ to node $i$. This weight-aware design enables more accurate estimation of shared parameters by prioritizing relevant signals from similar nodes.

We initialize the weight matrix $\Omega^k=I$. As the algorithm proceeds, the weights $\omega_{ji}^k$ are dynamically updated over time to reflect evolving similarities between nodes; see Algorithm \ref{alg:weight_update} for the update procedure.
Two key components are used to compute a new weight matrix $\Omega_{\text{new}}^k$.
\begin{itemize}
\item \textit{Arm‑selection similarity}, measured by the product $n_j^k n_i^k$, captures how frequently and consistently two nodes select the same arm, indicating stronger mutual relevance.
\item \textit{Context similarity}, measured by the cosine similarity between contextual features, assigns higher weights to node pairs with more similar contexts.
\end{itemize} 
After computing these components, all weights are normalized within the network to ensure proper scaling. To ensure computation stability, a smoothing process is applied at every time step, which is
\[
\Omega^k_t = \rho\,\Omega^k_{t-1} + (1-\rho)\,\Omega^k_{\text{new}},
\]
where $\rho \in [0,1)$ is a smoothing coefficient, which controls the influence of newly computed weights relative to historical ones.

\begin{algorithm}[htbp]
\caption{Weight update}
\label{alg:weight_update}
\textbf{Input}: Time step \(t\), arm selection count \(n_i^k\), contexts \(\mathbf{x}\), network \(\mathcal{G}\), smoothing coef $\rho$.\\
\vspace{-8mm}
\begin{algorithmic}[1]
\FORALL{arms \(k \in \mathcal{A}\)}
    \FORALL{nodes \(i, j \in \mathcal{G}\)}
        \STATE Compute arm selection similarity: $\frac{n_i^k n_j^k}{\left(\sum_{q \in \mathcal{G}} n_q^k\right)^2}$.
        \STATE Compute context similarity: $\frac{\mathbf{x}_{i,c}^\top \mathbf{x}_{j,c}}{\|\mathbf{x}_{i,c}\| \cdot \|\mathbf{x}_{j,c}\|}$.
    \ENDFOR
    \STATE Combine similarities and normalize to get $\Omega_{\text{new}}^k$.
    \STATE Update weight matrix: $\Omega^k \gets \rho  \Omega_{\text{old}}^k + (1 - \rho) \Omega_{\text{new}}^k$.
\ENDFOR
\end{algorithmic}
\textbf{Output}: Updated weight matrices \(\{\Omega^k\}_{k=1}^K\).
\end{algorithm}

\subsection{NetLinUCB}\label{sec:netlinucb}

To systematically analyze cross-node influences in decentralized contextual bandits, we propose NetLinUCB, a dynamic learning algorithm that incorporates shared contextual effects across nodes via adaptive weighted aggregation. Unlike Disjoint LinUCB, which treats each node independently, and Shared LinUCB, which pools all data globally, NetLinUCB enables partial information sharing while preserving node-level personalization.

Each node decomposes its context into a shared and node-specific part and follows standard ridge regression and UCB-based confidence bounds for local estimation. NetLinUCB then refines these estimates using parameter vectors shared from other nodes, weighted adaptively based on context similarity and arm-selection patterns. The network is assumed to be fully connected, and we investigate how network size affects performance. Arm selection proceeds by maximizing the upper confidence bound as defined in Definition~\ref{def:ucb-netlin}, and the complete algorithm is provided in Algorithm~\ref{alg:netlinucb_full}.

\textbf{Point estimate.} The initial parameter estimates in the Disjoint LinUCB (Appendix~\ref{app:disjoint_linucb}) are obtained via the closed‑form solution to ridge regression, \(
\hat{\theta}_t = \left(I_d + D_t^\top D_t \right)^{-1} D_t^\top z_t = W_t^{-1} b_t\), where \(W_t = I_d + D_t^\top D_t\) is the design matrix and \(b_t = D_t^\top z_t\) is the response vector.  
Here, \(D_t = \bigl[\mathbf{x}_{\tau}^\top\bigr]_{\tau \in \Psi_t} \in \mathbb{R}^{|\Psi_t| \times d}\) denotes the context history,  
\(z_t = \bigl[r_{\tau}\bigr]_{\tau \in \Psi_t} \in \mathbb{R}^{|\Psi_t| \times 1}\) denotes the reward history,  
and \(\Psi_t\) is the set of time steps prior to \(t\) at which arm \(a^{(k)}\) was chosen.
We further assume that the context vector is ordered such that the first \(d_c\) dimensions correspond to the common part, followed by the node‑specific part. Consequently, the first $d_c$ coordinates form $W_{i,c}^k$ and $b_{i,c}^k$, while the remaining $d_{i,s}$ coordinates form $W_{i,s}^k$ and $b_{i,s}^k$. Building on this formulation, NetLinUCB introduces inter‑node weighted aggregation to dynamically estimate parameters, i.e.,
\[
\hat{\theta}_{i,c}^{k} = \sum_{j=1}^{N}\omega_{ji}^{k}(W_{j,c}^{k})^{-1}b_{j,c}^{k}, \,
\hat{\theta}_{i,s}^{k} = (W_{i,s}^{k})^{-1} b_{i,s}^{k},
\]
where weight \(\omega_{ji}^{k}\) is updated according to Algorithm \ref{alg:weight_update}.

\begin{figure}[htbp]
  \centering
  \includegraphics[width=0.9\linewidth]{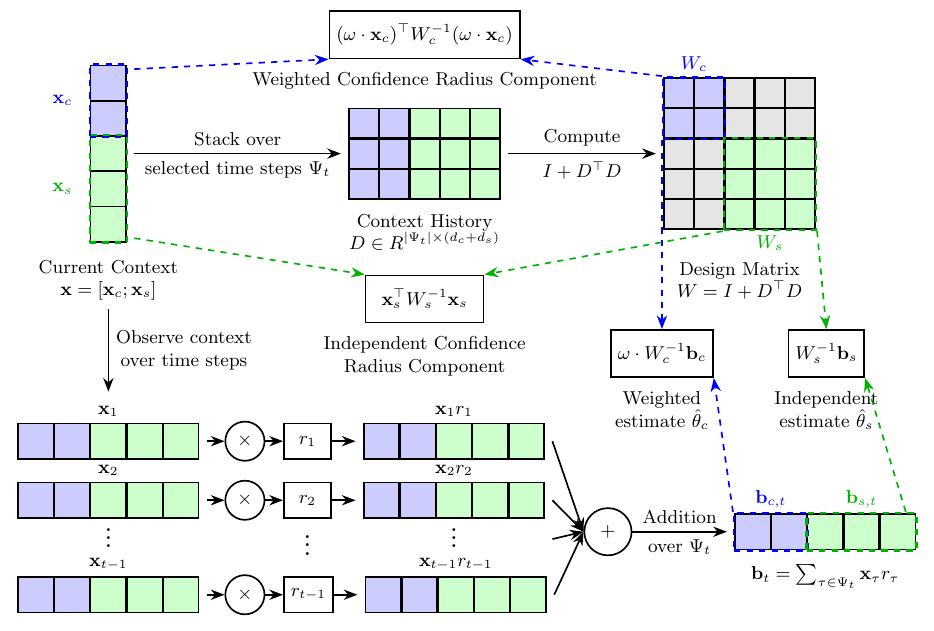} 
  \caption{Illustration of NetLinUCB estimation pipeline. The shared and specific context features \( \mathbf{x}_{c,t} \), \( \mathbf{x}_{s,t} \), and the historical design matrix \( D = [D_{c,t}, D_{s,t}] \) are aggregated to form a full design matrix \( W = I + D^\top D \). The resulting decomposition yields a weighted estimation of the shared parameter \( \hat{\theta}_c \) and a separate estimation of the specific component \( \hat{\theta}_s \), with associated confidence radius terms.}
  \label{fig:netlinucb_pipeline}
\end{figure}

\textbf{Confidence radius.} The confidence radius in Disjoint LinUCB is derived from the variance of the historical context, \(s_t = \sqrt{\mathbf{x}_t^\top W^{-1}\mathbf{x}_t}\). In NetLinUCB, we decompose \(W\) in the same manner as for \(\hat{\theta}\), and incorporate other nodes’ context histories. The confidence component becomes
\[
\sum_{j=1}^{N}(\omega_{ji}^{k})^2 \mathbf{x}_{i,c,t}^\top(W_{j,c}^{k})^{-1}\mathbf{x}_{i,c,t},
\]
where each weight \(\omega_{ji}^k\) is squared because the term 
\((W_{j,c}^{k})^{-1}\) represents a variance-level quantity; 
when aggregating variances, the weights contribute quadratically in contrast to the linear contribution in the point estimates. We also illustrate the estimation pipeline of NetLinUCB in Figure \ref{fig:netlinucb_pipeline}.

\begin{definition}
\label{def:ucb-netlin}
At each round $t$, for each arm $k$ and node $i$, NetLinUCB computes the upper confidence bound as
\[
\text{RIDGE-}\mathcal{U}_{i,t}^k = \mathbf{x}_{i,t}^\top [\hat{\theta}_{i,c}^{k}, \hat{\theta}_{i,s}^{k}] +
\alpha^{\text{ridge}} \cdot \sqrt{
\sum_{j=1}^{N}(\omega_{ji}^{k})^2\mathbf{x}_{i,c,t}^\top(W_{j,c}^{k})^{-1}\mathbf{x}_{i,c,t}
+ \mathbf{x}_{i,s,t}^\top(W_{i,s}^{k})^{-1}\mathbf{x}_{i,s,t}},
\] where $\alpha^{\text{ridge}}$ is the exploration parameter.
\end{definition}

Finally, the arm selection rule integrates both the shared and node-specific estimates along with their corresponding confidence radii, \(
a_{i,t} = \arg\max_{k\in\mathcal{A}} \mathrm{RIDGE}\text{-}\mathcal{U}_{i,t}^k\).

\begin{algorithm}[htbp]
\caption{NetLinUCB}
\label{alg:netlinucb_full}
\textbf{Input}: Number of agents $N$, arm set $\mathcal{A}$, context dimension $d_i = d_c + d_{i,s}$, exploration $\alpha^{\text{ridge}}$, network $\mathcal{G}$.\\
\textbf{Parameter}: Design matrix \( W_i^k \gets I_d \), response vector \( b_i^k \gets \mathbf{0}_d \), count matrix \( n \gets \mathbf{1} \in \mathbf{Z}_+^{N \times K}\), weight matrix \( \Omega^k\) for each arm $k$ and node $i$.
\begin{algorithmic}[1]
\FOR{each round $t = 1$ to $T$}
    \STATE \textbf{Update weights} via Algorithm~\ref{alg:weight_update}.
    \FORALL{nodes $i \in \mathcal{G}$}
        \STATE Observe context $\mathbf{x}_{i,t} = [\mathbf{x}_{i,c,t}; \mathbf{x}_{i,s,t}]$.
        \STATE Compute estimates and confidence radii.
        \STATE Select arm $a_{i,t} = \arg\max_{k \in \mathcal{A}} \ \text{RIDGE-}\mathcal{U}_{i,t}^k$.
        \STATE Observe reward $r_{i,t}$ for the selected arm $k \gets a_{i,t}$.
        \STATE Update $n_i^k$, $W_i^{k}$, and $b_i^{k}$.
    \ENDFOR
\ENDFOR
\end{algorithmic}
\end{algorithm}

\subsection{Net-SGD-UCB}\label{sec:sgd-ucb}

To address the scalability challenges of ridge regression with matrix inversion in high-dimensional settings, we propose Net-SGD-UCB. This method replaces closed-form updates with an online Stochastic Gradient Descent with Momentum (SGDM), combined with UCB exploration in a decentralized multi-node setting. Net-SGD-UCB selects the arm with the highest upper confidence bound as defined in Definition~\ref{def:ucb-netsgd} and we present the algorithmic details in Algorithm~\ref{alg:netsgducb_full}.

\textbf{Point estimate.} In classic momentum-based SGD, the gradient is computed using the standard squared loss function \(\mathcal{L}_{i,t}^k = \frac{1}{2}(r_{i,t}^k - \hat{r}_{i,t}^k)^2\). After calculating the gradient, parameter estimates \(\hat{\theta}\) are sequentially updated using momentum parameter $\mu$ and learning rate \(\eta^{\text{sgd}}\) as follows.
\begin{align*}
\nabla_{\theta}\mathcal{L}_{i,t}^k &= -(r_{i,t}^k - \mathbf{x}_{i,t}^\top \hat{\theta}_{i, t-1}^k)\mathbf{x}_{i,t}, \\
\mathbf{v}^{k}_{i,t} &= \mu \mathbf{v}^{k}_{i, t-1} + (1-\mu)\nabla_{\theta}\mathcal{L}_{i,t}^k, \\
\hat{\theta}^{k}_{i,t} &= \hat{\theta}^{k}_{i, t-1} - \eta^{\mathrm{sgd}} \mathbf{v}^{k}_{i,t}.
\end{align*}
Net-SGD-UCB then aggregates information from other nodes with weight matrix $\Omega^k$ to form common and node-specific components,
\[
\hat{\theta}_{i,c,t}^k = \sum_{j \in \mathcal{G}}\omega_{ji}^k \hat{\theta}_{j,c,t}^k, 
\,
\hat{\theta}_{i,s,t}^k = \hat{\theta}_{i,s,t}^k,
\]
where weight matrix $\Omega^k$ is updated based on Algorithm~\ref{alg:weight_update}.

\textbf{Confidence radius.} In standard SGD-UCB methods, e.g.,~\citet{zhou_neural_2020}, uncertainty is quantified by a geometry term $g^\top G g$, where $G$ accumulates historical gradient information. It captures the observed variation along different parameter directions. To adapt this idea to our context-aware network setting, we maintain only diagonal gradient accumulations. Specifically, we define $G_t = I + \sum_{\tau\in \Psi_t} \nabla_{\theta}\mathcal{L_\tau} (\nabla_{\theta}\mathcal{L_\tau})^\top$, where $\Psi_t$ denotes the set of time steps before time $t$ when the arm was selected. Because off-diagonal entries add limited value yet increase complexity when combined with the current context, we ignore them and update the diagonal elements,
\[
(G_t^k)_{hh} = \gamma (G_{t-1}^k)_{hh} + (1-\gamma)\left(\frac{\partial \mathcal{L}_{t}^k}{\partial \theta_h}\right)^2,
\]
where \( h \in \{1, \dots, d_c + d_s\} \), and \(\gamma\in [0, 1)\) is a smoothing parameter that balances historical and recent gradient information. Finally, we partition each diagonal matrix into a common component and a node‑specific component, forming $G_{j,c}^k$ and $G_{j,s}^k$, respectively.

\begin{definition}
\label{def:ucb-netsgd}
At round $t$, for arm $k$ and node $i$, Net-SGD-UCB computes the upper confidence bound as 
\[
\text{SGD-}\mathcal{U}_{i,t}^k = \mathbf{x}_{i,t}^\top [\hat{\theta}_{i,c}^{k}, \hat{\theta}_{i,s}^{k}] +
\alpha^{\text{sgd}} \cdot \sqrt{
\sum_{j=1}^{N}(\omega_{ji}^{k})^2\mathbf{x}_{i,c,t}^\top(G_{j,c,t}^{k})^{-1}\mathbf{x}_{i,c,t}
+ \mathbf{x}_{i,s,t}^\top(G_{i,s,t}^{k})^{-1}\mathbf{x}_{i,s,t}},
\] where $\alpha^{\text{sgd}}$ is the exploration parameter and $G_{j,c,t}^{k}$, $ G_{i,s,t}^{k}$ are EMA-smoothed diagonal gradient matrices.
\end{definition}

Finally, the upper confidence bound for arm selection balances exploitation with this adaptive uncertainty. Based on the criterion, we select the arm at time step $t$ as \(a_{i,t} = \arg\max_{k\in\mathcal{A}} \text{SGD-}\mathcal{U}_{i,t}^k\).

\begin{algorithm}[tb]
\caption{Net-SGD-UCB}
\label{alg:netsgducb_full}
\textbf{Input}: Number of agents $N$, arm set $\mathcal{A}$, context dimension $d_i = d_c + d_{i,s}$, learning rate $\eta^{\text{sgd}}$, momentum $\mu$, exploration parameter $\alpha^{\text{sgd}}$, EMA smoothing coefficient $\gamma$, network $\mathcal{G}$.\\
\textbf{Parameter}: Parameter estimates $\hat{\theta}_{i,t}^k$, momentum vectors $\mathbf{v}_{i,t}^k \gets \mathbf{0}$, weight matrix $\Omega^k$, and gradient accumulators $G_{i,t}^k \gets I_d$ for arm $k$ and node $i$.
\begin{algorithmic}[1]
\FOR{each round $t = 1$ to $T$}
    \STATE \textbf{Update weights} via Algorithm~\ref{alg:weight_update}.
    \FORALL{nodes $i \in \mathcal{G}$}
        \STATE Observe context $\mathbf{x}_{i,t} = [\mathbf{x}_{i,c,t}; \mathbf{x}_{i,s,t}]$.
        \STATE Compute estimates and confidence radii.
        \STATE Select arm $a_{i,t} = \arg\max_{k \in \mathcal{A}} \  \text{SGD-}\mathcal{U}_{i,t}^k$.
        \STATE Observe reward $r_{i,t}$ for the selected arm $k \gets a_{i,t}$.
        \STATE Update:  
            \begin{align*}
            \nabla_\theta \mathcal{L}_{i,t}^k
            &\gets -\bigl(r_{i,t}^k - \mathbf{x}_{i,t}^\top \hat{\theta}_{i,t-1}^k \bigr)\mathbf{x}_{i,t},\\
            \mathbf{v}_{i,t}^k
            &\gets \mu\,\mathbf{v}_{i,t-1}^k + (1-\mu)\,\nabla_\theta \mathcal{L}_{i,t}^k,\\
            \hat{\theta}_{i,t}^k
            &\gets \hat{\theta}_{i,t-1}^k - \eta^{\text{sgd}}\,\mathbf{v}_{i,t}^k,\\
            G_{i,t}^k &\gets \gamma G_{i,t-1}^k + (1-\gamma)\,\mathrm{diag}\bigl((\nabla_\theta \mathcal{L}_{i,t}^k)^{\odot 2}\bigr).
            \end{align*}
    \ENDFOR
\ENDFOR
\end{algorithmic}
\end{algorithm}

\section{Theoretical Performance Guarantee}\label{sec:regret-bound}

In this section, we derive theoretical performance guarantees (Definition~\ref{def:regret}) for the aforementioned algorithms, two newly proposed algorithms and two adapted benchmark algorithms. In Section \ref{sec:insight}, we further compare them in terms of computational complexity, communication cost, parameter sharing structure, and theoretical regret bounds.

\begin{figure}[t]
  \centering
  \includegraphics[width=0.85\textwidth]{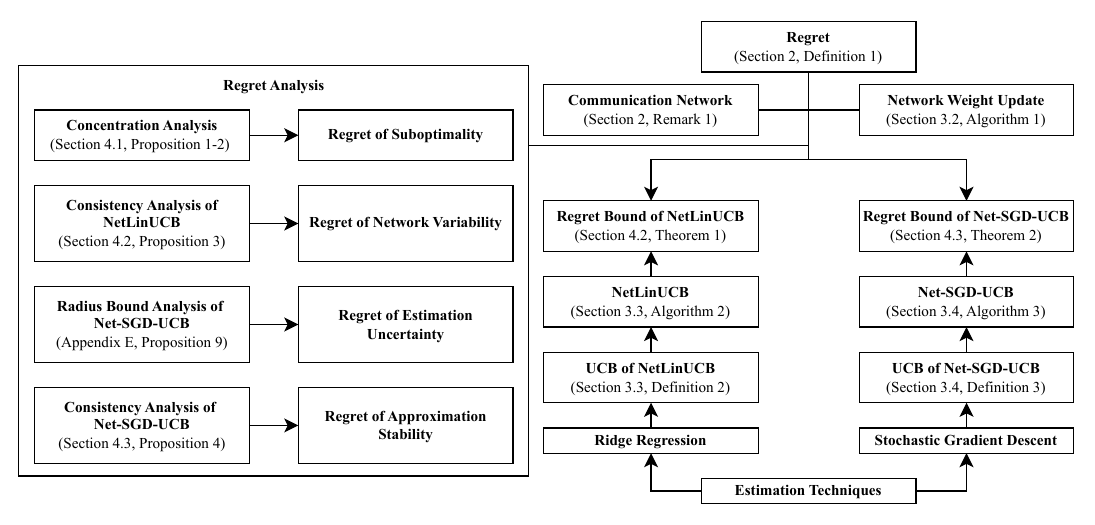}
  \caption{Roadmap of regret analysis.}
  \label{fig:roadmap}
\end{figure}

\subsection{Regret Analysis for Benchmarks}
We summarize the regret guarantees for the two benchmark algorithms. Specifically, we present their known bounds with respect to the time horizon $T$ and the network size $N$. We provide detailed proofs in Appendices~\ref{app:disjoint_linucb} and~\ref{app:whole_linucb}.

\begin{proposition}\label{prop:disjoint4} The total cumulative regret of Disjoint LinUCB over the network satisfies 
$$
R(T) = \sum_{i=1}^N R_i(T) = O\left( \sum_{i=1}^N d_i \sqrt{T \log T} \right).
$$
\end{proposition}

\begin{proposition}\label{prop:wholelinucb}
The total cumulative regret of Shared LinUCB over the network satisfies 
$$
R(T) = O\left((d_c + \sum_{i=1}^{N} d_{i,s}) \sqrt{T \log T} \right).
$$
\end{proposition}

These two benchmark algorithms represent two extremes. Shared LinUCB has a much lower regret by jointly exploiting all contextual information, but its effective dimension increases as \(d_c+\sum_i d_s\), resulting in costly matrix operations and high computational burden when the number of nodes \(N\) and the context dimensions are large. In contrast, Disjoint LinUCB is computationally efficient with lighter storage and simpler updates, but suffers from higher regret due to the lack of cross-node information sharing. Our goal is to strike a trade-off: achieving a cumulative regret of \(\tilde{\mathcal{O}}(\sqrt{T})\) while leveraging shared components to approach \(\sqrt{N}\) or \(\sqrt{d_c}\) scaling, without incurring the full computational burden of the shared model.

\subsection{Regret Analysis for NetLinUCB}\label{sec:netlinucb_regret}

The overall regret over the network can be naturally decomposed as the sum of the regret of each node in the network. We first provide high-probability bounds for each node and arm in Proposition \ref{prop:netlinucb-bound}. Then, based on these error bounds, we extend the single-node result to the entire network in Theorem \ref{thm:netlinucb-regret}, and yield one of the key theoretical results of this paper, the regret bound for NetLinUCB over the network.

We start with a specific node and a specific arm and analyze the confidence radius at time step $t \in [1, T]$. Without loss of generality, we replace the notation $\mathbf{x}_{i,c,t}$ with $\mathbf{x}_{c,t}$, $\theta_{i}$ with $\theta$, $\hat{\theta}_{i,t}$ with $\hat{\theta}_{t}$, and $\omega_{ji}^k$ with $\omega_{j}$ in Proposition \ref{prop:netlinucb-bound}. 

\begin{proposition}[Regret consistency under weighted confidence radius of NetLinUCB]\label{prop:netlinucb-bound} At each time step $t$, given historical information, independent contexts $\{\mathbf{x}_\tau\}_{\tau \in \Psi_t}$ and rewards $\{r_\tau\}_{\tau \in \Psi_t}$, we have
\begin{equation}
 \mathbb{P} \left( \left| \mathbf{x}_{c,t}^\top \tilde{\theta}_{t} - \mathbf{x}_{i,c,t}^\top \theta \right| \leq (\eta + 1) \tilde{s}_{t} + \zeta_t \right)   \geq  1 - \frac{1}{T},
\end{equation}
where \( \tilde{\theta}_{t} = \sum_{j=1}^N \omega_{j} \hat{\theta}_{j,t} \), \( \tilde{s}_{t} = \sqrt{\sum_{j=1}^N (\omega_{j} \mathbf{x}_{c,t})^\top W_{j,t}^{-1} (\omega_{j} \mathbf{x}_{c,t})} \), and $\zeta_t$ collects cross-term contributions with $\sum_{t=1}^T \zeta_t = O(1)$. Here, \(\eta = \sqrt{\log(2T)/2}\) ensures the stated high-probability guarantee via Hoeffding-type concentration.
\end{proposition}

We decompose the prediction error into variance and bias terms. The variance is bounded using Azuma’s inequality applied to a martingale sum over nodes, while the bias is controlled via the maximal spectral norm of historical context matrices. The full proof is provided in Appendix~\ref{prop:netlinucb-bound-proof}.

\begin{theorem}[Cumulative regret of NetLinUCB]\label{thm:netlinucb-regret}
The cumulative regret of the common component in the network is \(O\left(\sqrt{\frac{d_cT N}{c} \log (\frac{cTN}{d_c})}\right)\), where \( c > 0\) is a context‑diversity constant, defined as a lower bound on the minimum eigenvalue of the normalized context covariance matrix, and it is related to the common component dimension $d_c$. The cumulative regret of the node-specific component is \(
O\left( \sum_{i=1}^{N} d_{i, s} \sqrt{T \log T} \right)\). Thus, the overall cumulative regret \(R(T)\) of the NetLinUCB algorithm is \(O\left( \sqrt{\frac{d_cT N}{c} \log (\frac{cTN}{d_c})} + \sum_{i=1}^{N} d_{i, s} \sqrt{T \log T} \right)\).
\end{theorem}

\proof{Proof sketch of Theorem~\ref{thm:netlinucb-regret}.}
We split the cumulative regret into shared and node-specific parts. 
\begin{itemize}
    \item For the shared term, Proposition~\ref{prop:netlinucb-bound} provides a high-probability error bound. We use the lower bound on the minimum eigenvalue of 
\( W_{j,t} = I + \sum_{\tau \in \Psi_{j,t}^k} \mathbf{x}_{j,\tau} \mathbf{x}_{j,\tau}^\top \). The contexts span the feature space and therefore there exists a constant $c > 0$ such that
\(\lambda_{\min}\left(\tfrac{1}{|\Psi_{j,t}^k|}\sum_{\tau \in \Psi_{j,t}^k} \mathbf{x}_{j,\tau} \mathbf{x}_{j,\tau}^\top \right) \geq c\), which implies \(
\mathbf{x}^\top (W_{j,t})^{-1} x \leq \tfrac{d_c}{1+c\,|\Psi_{j,t}^k|}\). Let $n_t^k = \sum_{i=1}^N |\Psi_{j,t}^k|$. Regret terms can then be grouped by the chosen arm $k$, whenever the arm $k$ is selected $m = 1,\dots,n_T^k$ times, the cumulative contribution behaves as \(\sum_{m=1}^{n_T^k}\sqrt{\tfrac{d_c}{c\,m}}\), which follows the same summation structure as in classical bandit analysis and is bounded by \( O\left(\sqrt{\tfrac{d_c}{c}\,n_T^k \log n_T^k}\right)\). Summing over all arms with $\sum_{k} n_T^k = NT$ yields \( O\left(\sqrt{\tfrac{d_c T N}{c}\log(\frac{cTN}{d_c})}\right)\) for the shared part. 
\item The node-specific part reduces to \(N\) disjoint problems, each with regret \(O(d_{i,s}\sqrt{T\log T})\). 
\end{itemize}
We provide a full proof in Appendix~\ref{thm:netlinucb-proof}.
\QED \endproof

\subsection{Regret Analysis for Net-SGD-UCB}

The proof roadmap in this section parallels that of NetLinUCB in Section \ref{sec:netlinucb_regret}. We first establish a high-probability regret bound for each node in Proposition \ref{prop:sgd-ucb-regret-bound} and then aggregate the node-level results to obtain the overall regret over the network in Theorem \ref{thm:sgd-ucb-network-regret}.
To support Proposition \ref{prop:sgd-ucb-regret-bound}, we first derive a SGD-based confidence radius with adaptive weights in Proposition \ref{prop:grad-bound}, based on an accumulation bound for non-negative sequences in Lemma \ref{lemma:log-smooth-harmonic}. 

\begin{lemma}[Accumulation bound,~\citet{duchi_adaptive_2011}, Lemma 4]\label{lemma:log-smooth-harmonic}
Let \( \{a_t\}_{t=1}^T \) be a non-negative sequence with cumulative sum \( A_t = \sum_{\tau=1}^t a_\tau \). Then, \(\sum_{t=1}^T (a_t/A_t) \leq \log\left(1 + A_T/A_1\right) + 1\).
\end{lemma}

We consider the bound of the constructed confidence radius term for each arm \(k\) on a per‑node basis. 
For notational simplicity, we replace the context \( \mathbf{x}_{i,t} \) by \( \mathbf{x}_t \), 
the gradient accumulation matrix \( G_{i,t}^k \) by \( G_t \), estimates \(\hat{\theta}^{k}_{i, t}\) by \(\hat{\theta}_{t}\), and the sub‑Gaussian noise in the reward function \( \epsilon_{i,t}^k \) by \( \epsilon_t \).

\begin{proposition}[Confidence radius bound] \label{prop:grad-bound}
Under the EMA smoothing scheme and the MSGD setting, where the reward is a linear function with sub‑Gaussian noise $\epsilon$, the cumulative term \( \sum_t \mathbf{x}_t^\top G^{-1} \mathbf{x}_t \) is bounded as
\[
 \mathbf{x}_t^\top G^{-1} \mathbf{x}_t \leq \sum_{t=1}^T \frac{2d\log (\sigma^2 T)}{(1-\gamma)^2\sigma^2}+ \frac{(1-\mu)\sigma^2\log T}{(1+\mu)(1-\gamma)},
\]
where $d$ is the dimension of the contexts, $\sigma^2$ is the variance of the noise of our reward function, $\mu$ is the momentum parameter, $\gamma \in [0,1]$ is the smoothing parameter.
\end{proposition}

\begin{proof}{Proof sketch of Proposition~\ref{prop:grad-bound}}
We analyze the cumulative quadratic form under EMA smoothing. 
At each step, the diagonal of the gradient–accumulation matrix satisfies 
\(
(G_t)_{hh} = (1-\gamma)\sum_{s=1}^t \gamma^{t-s}\epsilon_s^2 (\mathbf{x}_s)_h^2 + \gamma^{t-1},
\)
which effectively smooths the past squared gradients with an exponential kernel.  
Using this structure, the key term 
\(
\sum_{t=1}^T \mathbf{x}_t^\top G_t^{-1}\mathbf{x}_t
\)
can be decomposed as a sum over coordinates. For each coordinate, 
Lemma~\ref{lemma:log-smooth-harmonic} bounds the harmonic sum 
\(
\sum_t \frac{\epsilon_t^2(\mathbf{x}_t)_h^2}{\sum_{s=1}^t \epsilon_s^2(\mathbf{x}_s)_h^2}
\)
by a logarithmic term, leading to a contribution on the order of 
\(
\frac{2d\log(\sigma^2 T)}{(1-\gamma)^2\sigma^2}.
\)
In parallel, we incorporate the effect of momentum updates.  
The momentum recursion accumulates past gradients as 
\(
v_{i,k}^{(t)} = (1-\mu)\sum_{s=1}^{t}\mu^{t-s-1}\nabla\mathcal{L}^{(s)},
\)
and a geometric–series argument together with the variance bound for sub‑Gaussian noise yields an additional contribution 
\(
\frac{(1-\mu)\sigma^2\log T}{(1+\mu)(1-\gamma)}.
\)

Adding both parts gives the stated upper bound. A full step‑by‑step proof, including all intermediate inequalities, is provided in Appendix~\ref{prop:grad-bound-proof}.
\end{proof}

Once the bound of the constructed confidence radius is obtained, the subsequent proposition demonstrates how this radius leads to consistent regret guarantees, providing a high-probability regret bound for each node.

\begin{proposition}[Regret consistency via radius propagation of Net-SGD-UCB]
\label{prop:sgd-ucb-regret-bound}
Under the conditions of Proposition~\ref{prop:grad-bound}, bounded contexts \( \mathbf{x}_{t} \), and sub-Gaussian noise \( \epsilon_t \) with variance \( \sigma^2 \). With a probability of at least \(1 - \frac{1}{T}\), the cumulative regret of Net-SGD-UCB in a single node is bounded by
\[
R(T) \leq 2\alpha^{\text{sgd}}\left( \frac{\sqrt{2dT \log (\sigma^2 T)}}{(1 - \gamma)\sigma} + \sigma \sqrt{\frac{(1 - \mu)T\log T}{(1 + \mu)(1 - \gamma)}} \right),
\]
where $\alpha^{\text{sgd}} \propto 1 + \sigma^2$ reflects the sensitivity to noise.
\end{proposition}

By combining the per-node regret guarantees with the inter-node weight update mechanism, we proceed to derive the cumulative regret bound over the entire network.

\begin{theorem}[Cumulative regret bound of Net-SGD-UCB]\label{thm:sgd-ucb-network-regret} In the Net-SGD-UCB framework, the cumulative regret decomposes into a shared and a node-specific component. The shared component scales as \(R_1(T) = O\left( \frac{\sqrt{ N d_c T \log (\sigma^2 T})}{1 - \gamma} \right)\). The node-specific regret satisfies \(R_2(T) = O \left( \sum_{i=1}^N \frac{\sqrt{2d_{i,s}T \log (\sigma^2 T)}}{(1 - \gamma)\sigma} + \sigma N \sqrt{\frac{(1 - \mu)T\log T}{(1 + \mu)(1 - \gamma)}} \right)\).
\end{theorem}

\proof{Proof sketch of Theorem~\ref{thm:sgd-ucb-network-regret}.}
We decompose the total regret into two parts: the shared regret from estimating the global parameter \(\theta_c\), and the node-specific regret associated with estimating local parameters \(\theta_{i,s}\). 
\begin{itemize}
    \item For the shared component, we analyze the estimation error dynamics of each node's shared parameter under weighted neighbor aggregation and stochastic gradient updates. The error consists of variance and bias terms. The bias vanishes since all nodes share the same true parameter \(\theta_c\), and the variance term is bounded by standard decentralized SGD analysis with EMA smoothing. As the number of rounds \(T\) grows, this variance diminishes. Therefore, the dominant contribution to regret arises from the weighted confidence radius, which we bound accordingly. 
    \item The node-specific component corresponds to running independent SGD-UCB instances in each node. By applying Proposition~\ref{prop:sgd-ucb-regret-bound} and summing the regret across all nodes, we obtain the final bound for this part. 
\end{itemize}

Combining both components yields the overall regret bound. The complete derivation is provided in Appendix~\ref{thm:sgd-ucb-network-regret-proof}.
\QED \endproof

\subsection{Comparisons and Insights}\label{sec:insight}

We compare our two proposed algorithms with two benchmark algorithms in Table~\ref{tab:model-comparison}.

\begin{table*}[htbp]
\centering
\setlength{\tabcolsep}{3pt}
\renewcommand{\arraystretch}{1.15} 
\small

\begin{tabular}{lcccc}
\toprule
\textbf{Algorithm} & \textbf{NetLinUCB} & \textbf{Net-SGD-UCB} & \textbf{Disjoint LinUCB} & \textbf{Shared LinUCB} \\
\midrule
\textbf{Update rule} & Weighted ridge & Momentum SGD & Per-node ridge & Global ridge \\
\textbf{Computation} & \( \mathcal{O}(NK(d_c+d_s)^3) \) & \( \mathcal{O}(NK(d_c+d_s)) \) & \( \mathcal{O}(NK(d_c+d_s)^3) \) & \( \mathcal{O}(K(d_c+Nd_s)^3) \) \\
\textbf{Memory} & \( \mathcal{O}(NK(d_c+d_s)^2) \) & \( \mathcal{O}(NK(d_c+d_s)) \) & \( \mathcal{O}(NK(d_c+d_s)^2) \) & \( \mathcal{O}(K(d_c+Nd_s)^2) \) \\
\textbf{Communication cost} & \( \mathcal{O}(N^2Kd_c) \) & \( \mathcal{O}(N^2Kd_c) \) & \( 0 \) & \( \mathcal{O}(N^2K(d_c+d_s)) \) \\
\textbf{Exploration} & Confidence radius & Adaptive covariant & Confidence radius & Confidence radius \\
\textbf{Shared parameters} & Weight matrix & Weight matrix & None & Full sharing \\
\textbf{Best use case} & Decentralized & High-dimensional & Static local & Centralized \\
\textbf{Regret bound} &
\shortstack[l]{%
\(\tilde{\mathcal{O}}(\sqrt{(d_cNT)/c} \)\\
\( + Nd_s\sqrt{T}) \)
} &
\shortstack[l]{%
\(\tilde{\mathcal{O}}(\sqrt{(d_c NT)/(1-\gamma)} \)\\
\( +N\sqrt{(d_{s} T)/(1-\gamma)})\)
} &
\shortstack[l]{%
\(\tilde{\mathcal{O}}(Nd_c\sqrt{T} \)\\
\( +Nd_s\sqrt{T})\)
} &
\shortstack[l]{%
\(\tilde{\mathcal{O}}(d_c \sqrt{T} \)\\
\( + N d_{s}\sqrt{T})\)
}\\
\bottomrule
\end{tabular}
\caption{Comparison of four contextual bandit models in a networked pricing.}
\label{tab:model-comparison}
\end{table*}

The computation and memory costs of Shared LinUCB scale poorly with the global context size, \(d_c+Nd_s\), due to the need for full matrix inversion. In contrast, NetLinUCB and Disjoint LinUCB incur only per-node cost. Net-SGD-UCB avoids matrix inversion entirely, yielding the most efficient complexity of \(\mathcal{O}(NK(d_c+d_s))\). Even in low-dimensional settings, the per-round runtime of Shared LinUCB remains several times higher than that of the other algorithms.

We also quantify the communication cost per round. Disjoint LinUCB requires no interaction, while NetLinUCB and Net-SGD-UCB exchange only shared parameter estimates, incurring a communication cost of \(\mathcal{O}(N^2 K d_c)\) per round across all arms. In contrast, Shared LinUCB synchronizes the full parameter vector, resulting in \(\mathcal{O}(N^2 K (d_c + d_s))\) communication per round.

NetLinUCB is suitable for decentralized networks because its regret improves with connectivity while preserving node-specific adaptation (Theorem~\ref{thm:netlinucb-regret}). Net-SGD-UCB performs best in high-dimensional or streaming environments due to its low computational complexity with SGD and diagonal variance (Theorem~\ref{thm:sgd-ucb-network-regret}). Disjoint LinUCB is effective when agents operate independently and model sizes are small, benefiting from isolated estimation and minimal memory demands (Proposition~\ref{prop:disjoint4}). Shared LinUCB excels in centralized environments with few nodes, where global coordination justifies the increased computational and memory cost (Proposition~\ref{prop:wholelinucb}).

Regret bound reveals that Disjoint LinUCB suffers from redundant shared learning, \(\tilde{\mathcal{O}}(Nd_c\sqrt{T})\), while Shared LinUCB reduces this to \(\tilde{\mathcal{O}}(d_c\sqrt{T})\). NetLinUCB further improves to \(\tilde{\mathcal{O}}(\sqrt{NT/c})\) through adaptive network averaging, and Net-SGD-UCB achieves similar gains via variance-aware updates. Both proposed methods better exploit structure for improved regret in decentralized systems. We conducted numerical experiments in Section~\ref{sec:numerical}, and the results align with our theoretical insights: models with structured information sharing achieve lower regret and faster convergence rates.

\section{Numerical Experiments}\label{sec:numerical}

We simulate a networked contextual bandit environment with synthetic instances generated as follows. Model parameters \( \theta_i^k \) are randomly initialized for each node. At each time step \( t \), context vectors are sampled from multivariate Gaussian distributions, with node-specific components drawn from local distributions to reflect heterogeneity.
As our interest lies in how regret scales with the learning horizon \( T \) and network size \( N \), we report two key metrics: the average per-round regret over time, \( R(t)/t \), and the average per-round per-node regret, \( R(T)/(NT) \).

\begin{figure}[htbp]
  \centering
  \includegraphics[width=0.7\columnwidth]{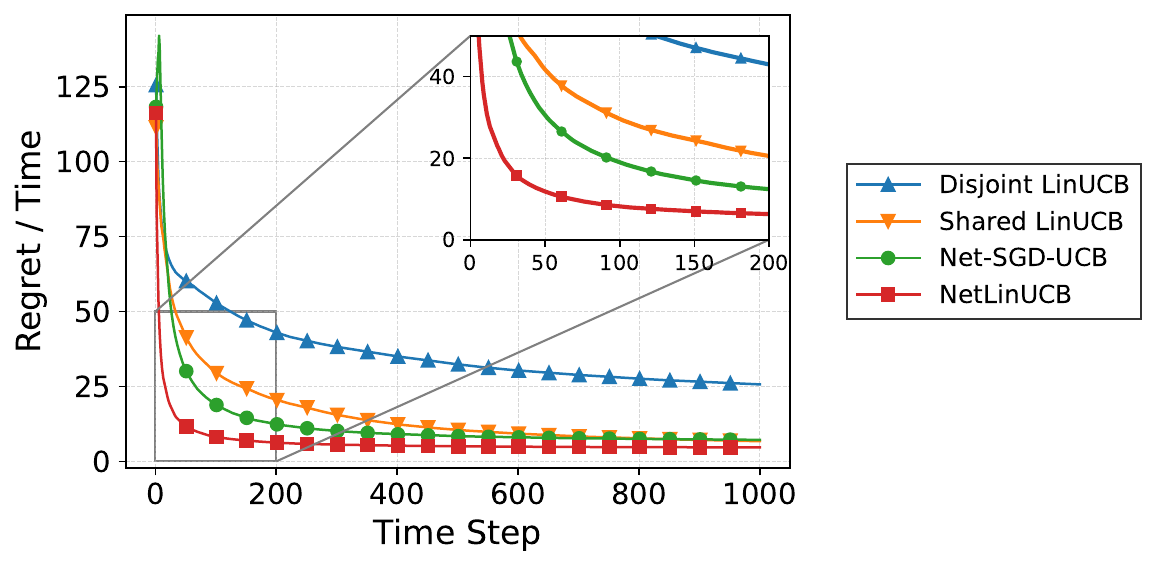}
  \caption{Time-average regret with \( T = 1000 \) and \( N = 12 \).}
\label{fig:exp1}
\end{figure}

\textbf{Sublinear regret in $T$.}  
Figure~\ref{fig:exp1} shows the per-round regret over time, $R(t)/t$. As the curves decrease towards zero, the two proposed algorithms achieve sublinear regret in $T$, matching the centralized benchmark while operating in a decentralized setting. Additional results across multiple instances are provided in Appendix~\ref{app:numerical}. Across all instances, network-based models outperform Disjoint LinUCB by effectively leveraging inter-city information. Shared LinUCB performs best when the proportion of shared context is extremely high and the reward gap between the optimal and suboptimal arms is large. This is because (i) a high proportion of shared context induces strong homogeneity across cities, making shared learning highly effective; and (ii) a large arm-reward gap enables the algorithm to quickly identify optimal arms using shared knowledge, thereby accelerating convergence. Despite its strong empirical performance in these regimes, Shared LinUCB incurs substantially higher per-round computational cost compared to the other methods. NetLinUCB benefits from strong connectivity and performs best when the proportion of node-specific context is extremely high, and arm-reward gaps are small, because it efficiently exploits shared structure to resolve fine-grained differences. Notably, Net-SGD-UCB demonstrates superior robustness in high-variance settings and scales well with larger action spaces due to its adaptive updates.

\begin{figure}[t]
  \centering
  \includegraphics[width=0.5\columnwidth]{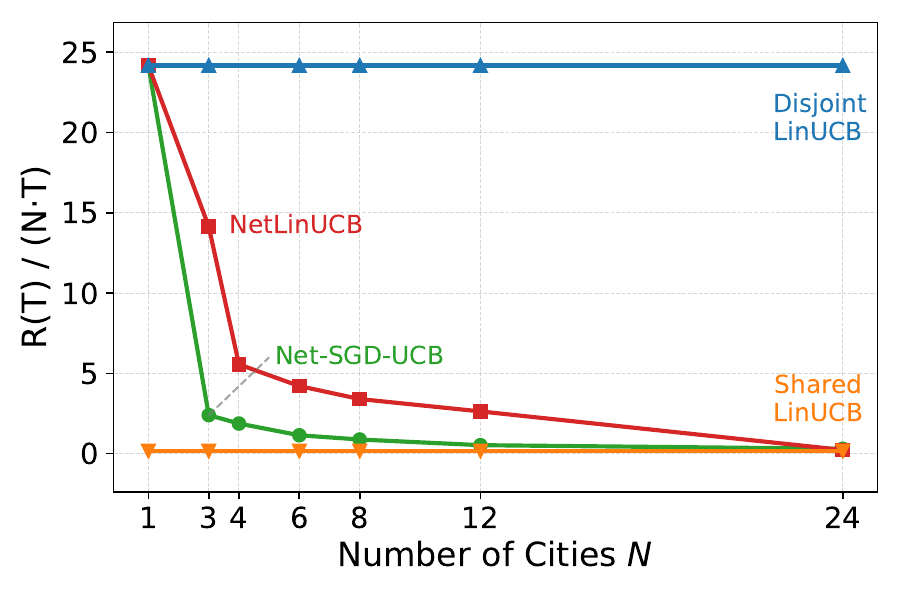}
  \caption{Per-node average regret under different networks.}
  \label{fig:exp6}
\end{figure}

\textbf{Sublinear regret on network size $N$.} 
We present the per-node regret over time \( R(t)/(NT)\)  of multiple instances in Figure \ref{fig:exp6}.
The results show that both NetLinUCB and Net-SGD-UCB achieve a per-node regret as low as the Shared LinUCB, demonstrating sublinear regret scaling in $N$. This highlights the scalability of network-aware methods in decentralized settings and validates the analysis in Section~\ref{sec:online-algo}.

\section{Conclusion}\label{sec:conclusion}

This work advances decentralized learning for structured decision-making, with broad applications in dynamic pricing, resource allocation, and distributed recommendation systems. We proposed two new UCB-based learning algorithms that achieve sublinear regret and reduce the regret associated with the shared feature space from linear scaling in the number of nodes $O(N)$ to sublinear $O(\sqrt{N})$. A key innovation of our approach lies in its use of adaptive weighting schemes to regulate inter-agent communication, allowing the network to transition between fully connected, partially connected, and even disjoint configurations depending on the agents’ informational relevance. Empirically, we demonstrate that increased network connectivity leads to lower regret, validating the benefits of inter-agent information sharing. 
These findings highlight the potential to improve efficiency in distributed markets while reducing reliance on centralized data, supporting scalability, privacy, and applicability in low-resource settings.

\vspace{.1in}
\def\bibfont{\scriptsize}
\bibliographystyle{informs2014}
\bibliography{references}

\clearpage

\begin{APPENDICES}

\section{Notation Summary}\label{app:notation}
\begin{table}[htbp]
\centering
\begin{tabular}{l|l}
\toprule
\textbf{Notation} & \textbf{Description} \\
\midrule
$a_{i,t}^{(k)} \in \mathcal{A}$ & Arm selected by node $i$ at time $t$; index $k=1,\ldots,K$; arm set $\mathcal{A}$\\
$b_{i,t}^k$ & Response vector in LinUCB parameter update \\
$c$ & Context diversity constant, lower bound on eigenvalue in LinUCB regret analysis \\
$d, d_i$ & Context dimension for node $i$; $d_i=d_c+d_{i,s}$\\
$d_c$ & Shared/common context dimension \\
$d_s, d_{i,s}$ & Node-specific context dimension for node $i$ \\
$g$ & Gradient used in SGD update \\
$h$ & Index for specific entry of vector or matrix \\
$i, j, q$ & Indices over nodes \\
$N$ & Total number of nodes, i.e., $i, j, q \in [N]$ \\
$\mathbf{x}_{i,t}$ & Full context vector for node $i$ at time $t$ \\
$\mathbf{x}_{i,c,t}$ & Shared context part for node $i$ at time $t$ \\
$\mathbf{x}_{i,s,t}$ & Node-specific context part for node $i$ at time $t$ \\
$\epsilon$ & Sub-Gaussian noise with zero mean and variance $\sigma^2$ \\
$r$ & Reward, modeled as a linear function of context \\
$\theta$ & Model parameters; include shared \(\theta_c\) and local \(\theta_{i,s}\) components \\
$t, \tau$ & Time step indices \\
$\mathbf{v}$ & Momentum vector in SGD \\
$G_{i,t}^k$ & Diagonal matrix accumulating squared gradients (SGD) for node $i$, arm $k$ \\
$W_{i,t}^k$ & Design matrix storing historical contexts in LinUCB\\
\( s_t \) & Confidence radius in standard LinUCB, \( s_t = \sqrt{\mathbf{x}_t^\top W_t^{-1} \mathbf{x}_t} \) \\
\( \tilde{s}_t \) & Weighted confidence radius in NetLinUCB, \( \tilde{s}_t = \sqrt{\sum_{j=1}^N (\omega_j \mathbf{x}_{c,t})^\top W_{j,t}^{-1} (\omega_j \mathbf{x}_{c,t})} \) \\
$\mathcal{G}$ & Fully connected network of nodes \\
$\Omega^k$ & Weight matrix for arm $k$ over network $\mathcal{G}$ \\
$\omega_{ji}^k$ & Influence of node $j$ on node $i$ for arm $k$ \\
$\gamma$ & EMA smoothing parameter \\
$\eta$ & Confidence level parameter in LinUCB regret analysis, $\eta = \sqrt{\frac{1}{2} \log(2T)}$ \\
$\alpha^{\text{ridge}}$ & Exploration parameter in NetLinUCB, scaling the confidence radius \\
$\alpha^{\text{sgd}}$ & Exploration parameter in Net-SGD-UCB, scaling the confidence radius\\
$\eta^{\text{sgd}}$ & Learning rate in SGD update \\
$\mu$ & Momentum coefficient in SGD \\
$n_{i,t}^k$ & Number of times arm \(k\) has been selected by node \(i\) before time \(t\) \\
$n_t^k$ & Total number of times arm \(k\) has been selected by all nodes before time \(t\) \\
$\Psi_{i,t}^k$ & Set of time steps before \(t\) when arm \(k\) was selected; \(|\Psi_{i,t}^k| = n_{i,t}^k\) \\
RIDGE-$\mathcal{U}_{i,t}^k$ & UCB expression in NetLinUCB using design matrix \(W\) \\
SGD-$\mathcal{U}_{i,t}^k$ & UCB expression in Net-SGD-UCB using gradient-based matrix \(G\) \\
\bottomrule
\end{tabular}
\caption{Summary of Notation}
\label{tab:notation}
\end{table}

\section{Supporting Results for Disjoint LinUCB}\label{app:disjoint_linucb}
This section complements the benchmark algorithms presented in Section~\ref{sec:benchmark}. The Disjoint LinUCB algorithm \citep{chu_contextual_2011, abbasi-yadkori_improved_2011} assumes that, on each node, the coefficients for each arm \(a^{(k)}\) are independent, and that context vectors \(\mathbf{x}_t\) are shared across all arms. 
\subsection{Formulation for Disjoint LinUCB}\label{sec:disjoint_formulation}
The expected reward of arm \(a^{(k)}\) is given by \(\mathbb{E}[r_t^k \mid \mathbf{x}_{t}] = \mathbf{x}_{t}^\top \mathbf{\theta}^k\). The observed reward of arm \(a^{(k)}\) is modeled as \(r_t^k = \mathbf{x}_t^\top \mathbf{\theta}^k + \epsilon_k\), where \(\mathbf{\theta}^k\) is the true coefficient vector for arm \(a^{(k)}\), and \(\epsilon_k \sim \mathcal{N}(0, \sigma_k^2)\) is Gaussian noise. At time \(t\), the optimal arm is \(a_t^{\star} = \arg\max_{a \in \mathcal{A}} \mathbf{x}_t^\top \mathbf{\theta}^{a}\), while the agent selects \(a_t = \arg\max_{a \in \mathcal{A}} \mathbf{x}_t^\top \hat{\mathbf{\theta}}^a\) based on estimated parameters. After receiving the corresponding reward, the agent updates its historical data. The cumulative regret over \(T\) rounds is defined as
\( R(T) = \mathbb{E}[\sum_{t=1}^T r_{t, a_t^{\star}}] - \mathbb{E}[\sum_{t=1}^T r_{t, a_t}] = \sum_{t=1}^T \mathbf{x}_t^\top \mathbf{\theta}^{a_t^{\star}} - \sum_{t=1}^T \mathbf{x}_t^\top \mathbf{\theta}^{a_t}\).

For each arm \(k\), let \(s_t = \sqrt{\mathbf{x}_t^\top W_{t}^{-1} \mathbf{x}_t} \in \mathbb{R}^+\) denote the standard deviation term at round \(t\). Let \(D_{t} = \left[\mathbf{x}_{\tau}^\top \right]_{\tau \in \Psi_t} \in \mathbb{R}^{|\Psi_t| \times d}\) and \(z_t = [r_{\tau}]_{\tau \in \Psi_t} \in \mathbb{R}^{|\Psi_t| \times 1}\) represent the context and reward histories respectively, where \(\Psi_t\) is the set of time steps \(\tau\) when arm \(k\) was selected before \(t\). We define the design matrix \(W_{t} = I_d + D_{t}^\top D_{t}\) and the response vector \(b_{t} = D_{t}^\top z_t\). LinUCB uses ridge regression to estimate the coefficient vector via the optimization \(\hat{\theta} = \min_{\left\| \theta \right\| \le 1} \left\| W\theta-b \right\|^2 + \left\|  I \theta \right\|^2\), where the regularization parameter is set to $\lambda=1$. It yields the closed-form solution,
\[
\hat{\theta} = (I_d + D_{t}^\top D_{t})^{-1}D_t^\top z_t = W_{t}^{-1} b_t.
\]

LinUCB estimates the reward for each arm by combining the context vector \(\mathbf{x}_t\) with the estimated parameter vector \(\hat{\mathbf{\theta}}^k_t\). To capture uncertainty, it incorporates a confidence term scaled by an exploration parameter \(\alpha^{\text{ridge}}\). At each time step, the algorithm selects the arm with the highest upper confidence bound:
\[\hat{r}_{t}^k = \mathbf{x}_t^\top \hat{\theta}_t^k + \alpha^{\text{ridge}} \sqrt{\mathbf{x}_t^\top W^{-1} \mathbf{x}_t}.\]
This approach balances exploration, selecting arms with greater uncertainty, and exploitation, choosing arms expected to yield higher rewards. LinUCB is computationally efficient and easy to implement, making it well-suited for settings where the linear reward assumption holds.

\begin{algorithm}[ht]
\caption{Disjoint LinUCB}
\label{alg:disjoint-linucb}
\textbf{Input}: Regularization parameter \(\lambda\), exploration parameter \(\alpha^{\text{ridge}}\), time horizon \(T\).\\
\textbf{Parameter}: Design matrix \(W_k = I_d\), and response vector \(b_k = \mathbf{0}_d\) for each arm \(k \in \mathcal{A}\).
\begin{algorithmic}[1]
\FOR{\(t = 1\) to \(T\)}
    \STATE Observe context \(\mathbf{x}_t \in \mathbb{R}^d\).
    \FOR{each arm \(k \in \mathcal{A}\)}
        \STATE Compute estimate: \(\hat{\theta}_t^k \gets W_k^{-1} b_k\).
        \STATE Compute UCB: \(\hat{r}_t^k \gets \mathbf{x}_t^\top \hat{\theta}_t^k + \alpha^{\text{ridge}} \sqrt{\mathbf{x}_t^\top W_k^{-1} \mathbf{x}_t}\).
    \ENDFOR
    \STATE Select arm \(a_t = \arg\max_{k} \hat{r}_t^k\); observe reward \(r_t\).
    \STATE Update:
    \begin{align*}
    W_{a_t} &\gets W_{a_t} + \mathbf{x}_t \mathbf{x}_t^\top,\\
    b_{a_t} &\gets b_{a_t} + r_t \mathbf{x}_t.
    \end{align*}
\ENDFOR
\end{algorithmic}
\end{algorithm}

\subsection{Analysis for Disjoint LinUCB}\label{sec:disjoint_analysis}
We analyze the regret of the Disjoint LinUCB algorithm following the formulation and results from \citet{chu_contextual_2011, abbasi-yadkori_improved_2011}.

\begin{proposition}[Error bound for Disjoint LinUCB, \citet{chu_contextual_2011}, Lemma 1] \label{prop:disjoint1}
Suppose for each node \(i \in [N]\), the input index set \(\Psi_t\) in Disjoint LinUCB is constructed such that, for any fixed context vector \( \mathbf{x}_{i,\tau} \) with \( \tau \in \Psi_{i,t} \), the corresponding reward variables \( r_{i,\tau} \) are independent random variables with expectations \( \mathbb{E}[r_{i,\tau}] = \mathbf{x}_{i,\tau}^\top \theta_i \). Then, with probability at least \( 1 - \frac{1}{T} \), the following inequality holds uniformly over all arms \( a \in \mathcal{A} \) and nodes \(i \in [N]\):
\[
\left| \mathbf{x}_{i,t}^\top \hat{\theta}_i^a -  \mathbf{x}_{i,t}^\top \theta_i^a \right| \leq (\eta + 1) s_{i,t}.
\]
\end{proposition}

\proof{Proof of Proposition \ref{prop:disjoint1}.}
Using the standard notation from the Disjoint LinUCB algorithm, we can express the prediction error as
\begin{align*}
\mathbf{x}_{i,t}^\top\hat{\theta}_i -  \mathbf{x}_{i,t}^\top \theta_i &= \mathbf{x}_{i,t}^\top W_{i,t}^{-1} b_{i,t} - \mathbf{x}_{i,t}^\top W_{i,t}^{-1} \left(I_d + D_{i,t}^\top D_{i,t} \right) \theta_i \\
&= \mathbf{x}_{i,t}^\top W_{i,t}^{-1} D_{i,t}^\top r_{i,t} - \mathbf{x}_{i,t}^\top W_{i,t}^{-1} \left( \theta_i + D_{i,t}^\top D_{i,t} \theta_i \right)\\
&= \mathbf{x}_{i,t}^\top W_{i,t}^{-1} D_{i,t}^\top (r_{i,t} - D_{i,t} \theta_i) - \mathbf{x}_{i,t}^\top W_{i,t}^{-1} \theta_i.
\end{align*}
Because of \( \| \theta_i \| \leq 1 \) by Assumption~\ref{assump:bounded}, we have
\[
\left| \mathbf{x}_{i,t}^\top\hat{\theta}_i -  \mathbf{x}_{i,t}^\top \theta_i \right| \leq \left| \mathbf{x}_{i,t}^\top W_{i,t}^{-1} D_{i,t}^\top (r_{i,t} - D_{i,t} \theta_i) \right| + \left\| W_{i,t}^{-1} \mathbf{x}_{i,t} \right\|.
\]
The right-hand side above decomposes the prediction error into a variance term (first) and a bias term (second). Because the data indexed in \(\Psi_{i,t} \) are independent, \(E[r_{i,t} - D_{i,t} \theta_i] = 0\). Applying Hoeffding's inequality yields
\[
\Pr\left( \left| \mathbf{x}_{i,t}^\top W_{i,t}^{-1} D_{i,t}^\top (r_{i,t} - D_{i,t} \theta_i) \right| > \eta_i s_{i,t} \right) \leq 2 \exp \left( -\frac{2 \eta^2 s^2_{i,t}}{ \left\| D_{i,t} W_{i,t}^{-1} \mathbf{x}_{i,t} \right\|^2} \right) \leq 2 \exp(-2 \eta^2) =  \frac{1}{T},
\]
where the last inequality follows from the fact that
\[
s^2_{i,t} = \mathbf{x}_{i,t}^\top W_{i,t}^{-1} \mathbf{x}_{i,t} = \mathbf{x}_{i,t}^\top W_{i,t}^{-1} (I_d + D_{i,t}^\top D_{i,t}) W_{i,t}^{-1} \mathbf{x}_{i,t} 
\ge \mathbf{x}_{i,t}^\top W_{i,t}^{-1} D_{i,t}^\top D_{i,t} W_{i,t}^{-1} \mathbf{x}_{i,t} = \left( D_{i,t} W_{i,t}^{-1} \mathbf{x}_{i,t} \right)^2.
\]

Setting \( \eta_i = \sqrt{\frac{\log(2T)}{2}} \) ensures that the tail probability is bounded by \( \frac{1}{T} \), i.e.,
\[
2 \exp(-2 \eta_i^2) = \frac{1}{T}.
\]

Applying a union bound over all arms, we can guarantee, with probability at least \( 1 - 1 / T \), that for all arms \( a \in \mathcal{A} \),
\[
\left| \mathbf{x}_{i,t}^\top W_{i,t}^{-1} D_{i,t}^\top (r_{i,t} - D_{i,t} \theta_i) \right| \leq \eta_i s_{i,t}.
\]

We now turn to bounding the bias term, where
\[
\left\| W_{i,t}^{-1} \mathbf{x}_{i,t} \right\| = \sqrt{\mathbf{x}_{i,t}^\top W_{i,t}^{-1} I_d W_{i,t}^{-1} \mathbf{x}_{i,t}} 
\leq \sqrt{\mathbf{x}_{i,t}^\top W_{i,t}^{-1} (I_d + D_{i,t}^\top D_{i,t}) W_{i,t}^{-1} \mathbf{x}_{i,t}} 
= s_{i,t}.
\]
Combining the variance and bias bounds yields the final result.
\QED \endproof

\begin{proposition} [Determinant-Trace Inequality, \citet{auer_finite-time_2002}, Lemma 11] \label{prop:disjoint2} 
Let \( \mathbf{x}_1, \mathbf{x}_2, \dots, \mathbf{x}_t \in \mathbb{R}^d \) with \( \| \mathbf{x}_\tau \|_2 \leq L \) for all \( 1 \leq \tau \leq t \), and let \(\lambda >0\). Define \(W_t = \lambda I + \sum_{\tau=1}^T \mathbf{x}_{\tau} \mathbf{x}_{\tau}^\top\). Then,
\[
\det(W_t) \leq \left( \lambda + \frac{t L^2}{d} \right)^d.
\]
\end{proposition}

\proof{Proof of Proposition \ref{prop:disjoint2}.}
Let \( \lambda_1, \lambda_2, \dots, \lambda_d\) denote the eigenvalues of \( W_t \). Because \( W_t \) is symmetric positive definite, all eigenvalues are strictly positive. Note that \(\det(W_t) = \prod_{h=1}^d \lambda_h \), and \( \operatorname{tr}(W_t) = \sum_{h=1}^d \lambda_h \). By the inequality of arithmetic and geometric means, we have
\[
\prod_{h=1}^d \lambda_h \leq \left( \frac{1}{d} \sum_{h=1}^d \lambda_h\right)^d.
\]
We now bound the trace,
\[
\operatorname{tr}(W_t) = \operatorname{tr}(I) + \sum_{\tau=1}^T \operatorname{tr}(\mathbf{x}_{\tau} \mathbf{x}_{\tau}^\top) = d \lambda + \sum_{\tau=1}^T \| \mathbf{x}_{\tau} \|_2^2 \leq d \lambda + t L^2.
\]
Combining the bounds for the determinant and trace, we obtain
\[
\det(W_t) \leq \left( \lambda + \frac{t L^2}{d} \right)^d.
\]
\QED \endproof

\begin{proposition}[Log-Determinant growth, \citet{abbasi-yadkori_improved_2011}, Theorem 1] \label{prop:disjoint3} 
Let \(W_t =  \lambda I + \sum_{\tau=1}^{t} \mathbf{x}_{\tau} \mathbf{x}_{\tau}^\top\), then the following inequality holds:
\[
\log \left( \frac{\det(W_T)}{\det(W_0)} \right) \leq \sum_{t=1}^T \mathbf{x}_t^\top W_{t-1}^{-1} \mathbf{x}_t .
\]
Moreover, if \( \|\mathbf{x}_t\|^2 \leq L \) for all \( t \), then
\[
\sum_{t=1}^T \min \{ 1, \mathbf{x}_t^\top W_{t-1}^{-1} \mathbf{x}_t \} \leq 2 \log \left( \frac{\det(W_T)}{\det(W_0)} \right) \leq 2 \left( d \log \left( \lambda + \frac{TL^2}{d} \right) - d \log \lambda \right).
\]
Finally, if \( \lambda \geq \max(1, L^2) \), then
\[
\log \left( \frac{\det(W_T)}{\det(W_0)} \right) \leq \sum_{t=1}^T \mathbf{x}_t^\top W_{t-1}^{-1} \mathbf{x}_t \leq 2 \log \left( \frac{\det(W_T)}{\det(W_0)} \right) \leq 2d \log \left( 1 + \frac{TL^2}{d\lambda} \right).
\]
\end{proposition}

\proof{Proof of Proposition \ref{prop:disjoint3}.}
We first apply the matrix determinant lemma,
\[
\det(W_T) = \det(W_{T-1} + \mathbf{x}_t \mathbf{x}_t^\top) = (1+ \mathbf{x}_t^\top W_{T-1}^{-1} \mathbf{x}_t) \det(W_{T-1}).
\]
Taking the logarithm on both sides yields
\[
\log \det(W_T) = \log(1+ \mathbf{x}_t^\top W_{T-1}^{-1} \mathbf{x}_t) + \log\det(W_{T-1}).
\]
Because \(\log(1 + x) \leq x \) for all \(x \geq 0\), we obtain
\[
\log \det(W_T) \leq \log \det(W_0) + \sum_{t=1}^T \mathbf{x}_t^\top W_{t-1}^{-1} \mathbf{x}_t.
\]
Using the inequality \( x \leq 2 \log(1 + x)\) for \(x \in [0, 1]\), and the previous inequality, we obtain
\[
\sum_{t=1}^T \min \{ 1, \mathbf{x}_t^\top W_{t-1}^{-1} \mathbf{x}_t \} \leq 2 \sum_{t=1}^T \log \left( 1 + \mathbf{x}_t^\top W_{t-1}^{-1} \mathbf{x}_t \right) = 2 \log \left( \frac{\det(W_T)}{\det(W_0)} \right).
\]

To bound the log-determinant radio, we invoke Proposition~\ref{prop:disjoint2}, which gives \[\log \det(W_T) \leq d \log \left( \lambda +\frac{ TL^2}{d} \right).\]
Because \(W_0 = \lambda I\), we have 
\[
\log \left( \frac{\det(W_T)}{\det(W_0)} \right) \leq d \log \left( \frac{\lambda + \frac{T L^2}{d}}{\lambda} \right) = d \log \left( 1 + \frac{T L^2}{d\lambda} \right),
\]
finishing the proof of the second inequality. 

Lastly, we show that the individual quadratic terms are uniformly bounded if \(\frac{L^2}{\lambda} \leq 1\). The sum \(\sum_{t=1}^T \mathbf{x}_t^\top W_{t-1}^{-1} \mathbf{x}_t\)
can itself be upper bounded as a function of \( \log \det(W_T) \) provided that \(\lambda \) is large enough. Using Matrix Inversion proposition, we know
\[
\mathbf{x}_t^\top W_{t-1}^{-1} \mathbf{x}_t \leq \frac{\|\mathbf{x}_{t}\|^2}{\lambda}  \leq \frac{L^2}{\lambda}.
\]
Hence, we get that if \(\lambda \geq \max(1, L^2) \),
\[
\log \left( \frac{\det(W_T)}{\det(W_0)} \right) \leq \sum_{t=1}^T \mathbf{x}_t^\top W_{t-1}^{-1} \mathbf{x}_t \leq 2 \log \left( \frac{\det(W_T)}{\det(W_0)} \right) \leq 2d \log \left( 1 + \frac{TL^2}{d\lambda} \right).
\]
\QED \endproof

By gathering the technical analysis developed in the above propositions, we derive the regret bound for the $N$-node Disjoint LinUCB algorithm stated in the main text. Below we provide a concise proof sketch.

\proof{Proof sketch of Proposition~\ref{prop:disjoint4}.}
The proof proceeds by summing per-node regret bounds across all $N$ nodes. For each node $i$, we apply:
\begin{itemize}
    \item[(i)] the high-probability prediction error bound from Proposition~\ref{prop:disjoint1}, using a unified confidence parameter \(\eta_i\),
    \item[(ii)] the determinant-trace bound from Proposition~\ref{prop:disjoint2} to control the growth of the local design matrix \(W_{i,t}\),
    \item[(iii)] the cumulative variance bound from Proposition~\ref{prop:disjoint3}: 
    \[
    \sum_{t=1}^T \mathbf{x}_{i,t}^\top W_{i,t}^{-1} \mathbf{x}_{i,t} \leq 2 d_i \log\left(1 + \frac{T}{\lambda d_i} \right).
    \]
\end{itemize}
Using the Cauchy–Schwarz inequality, the per-node regret can be bounded as \(
R_i(T) = \sum_{t=1}^T \left( \mathbf{x}_{i,t}^\top \theta_i^{a_t^*} - \mathbf{x}_{i,t}^\top \theta_i^{a_t} \right) 
\leq (1 + \eta_i) \sqrt{T \cdot \sum_{t=1}^T \mathbf{x}_{i,t}^\top W_{i,t}^{-1} \mathbf{x}_{i,t}} 
= O\left( d_i \sqrt{T \log T} \right)
\). Summing over all nodes gives the network-level cumulative regret: 
\[
R(T) = \sum_{i=1}^N R_i(T) = O\left( \sum_{i=1}^N d_i \sqrt{T \log T} \right),
\]
as claimed.
\QED \endproof

\begin{remark}
This result confirms that LinUCB achieves sub-linear regret in $T$ for each node, ensuring that the algorithm effectively balances exploration and exploitation through the decision-making process.
\end{remark}

\section{Supporting Results for Shared LinUCB}\label{app:whole_linucb}
This section supplements the benchmark algorithms described in Section~\ref{sec:benchmark}. To capture inter-node dependencies, we adpot a unified framework that embeds both shared and node-specific features within a global parameter space, treating all nodes within a single Disjoint LinUCB model. The shared parameters $\mathbf{\theta}^k_c$ are initialized uniformly across all nodes, ensuring that they consistently reflect shared contextual influences. As a result, this formulation enhances predictive power and coherence across nodes.
\subsection{Formulation for Shared LinUCB}\label{sec:whole_formulation}
We define a unified feature space of dimension \(d = d_c + \sum_{i=1}^{N} d_{i,s}\), where \(d_c\) corresponds to the shared feature dimension and \(d_{i,s}\) denotes local features dimension of node \(i\). Each context vector is embedded as \(\mathbf{x}_t = [\mathbf{x}_c, \mathbf{x}_{1,s}, \dots, \mathbf{x}_{N,s}] \in \mathbb{R}^d\), and arm-specific parameters are organized as \(\mathbf{\theta}^k = [\mathbf{\theta}_c^k, \mathbf{\theta}_{1,s}^k, \dots, \mathbf{\theta}_{N,s}^k]\). For node \(i\), only the corresponding shared and local features are active, the expected reward is computed as
\[
\mathbb{E}[r_{i, t}^k \mid \mathbf{x}_{i, t}] = [\mathbf{x}_{i,c}, 0, 0, \ldots, 0, \mathbf{x}_{i,s}, 0, \ldots, 0]^\top \mathbf{\theta}^k,
\]
where the zero blocks deactivate irrelevant node features.

\begin{algorithm}[ht]
\caption{Shared LinUCB}
\label{alg:whole_linucb}
\textbf{Input}: Regularization parameter \(\lambda\), exploration parameter \(\alpha^{\text{ridge}}\), number of nodes \(N\), number of arms \(K\), total dimension \(d = d_c + \sum_{i=1}^N d_{i,s}\), time horizon $T$.\\
\textbf{Parameter}: Design matrix \(W_k = I_d \in \mathbb{R}^{d \times d}\) and response vector \(b_k = \mathbf{0}_d \in \mathbb{R}^d\) for each arm \(k\).
\begin{algorithmic}[1]
\FOR{each time step \(t = 1\) to \(T\)}
    \FOR{each node \(i = 1\) to \(N\)}
        \STATE Construct global context vector \(\mathbf{x}_{i,t} \in \mathbb{R}^d\): \(\mathbf{x}_{i,t} = [\mathbf{x}_{c,t}, 0, \ldots, 0, \mathbf{x}_{i,s,t}, 0, \ldots, 0]^\top\).
        \STATE Compute estimated parameter: \(\hat{\theta}_k \gets W_k^{-1} b_k\) for each arm \(k\). 
        \STATE Select arm \(a_{i,t} = \arg\max_k \mathbf{x}_{i,t}^\top \hat{\theta}_k + \alpha^{\text{ridge}} \cdot \sqrt{ \mathbf{x}_{i,t}^\top W_k^{-1} \mathbf{x}_{i,t} }\); observe reward \(r_{i,t}\).
        \STATE Update:
        \begin{align*}
        W_{a_{i,t}} &\gets W_{a_{i,t}} + \mathbf{x}_{i,t} \mathbf{x}_{i,t}^\top,\\
        b_{a_{i,t}} &\gets b_{a_{i,t}} + r_{i,t} \cdot \mathbf{x}_{i,t}.
        \end{align*}
    \ENDFOR
\ENDFOR
\end{algorithmic}
\end{algorithm}

\subsection{Analysis for Shared 
LinUCB}\label{sec:whole_analysis}
Shared LinUCB can be viewed as a variant of Disjoint LinUCB that aggregates all nodes into a single multi‑armed bandit model. Below, we present a concise proof sketch of the regret bound for Shared LinUCB (Proposition~\ref{prop:wholelinucb}).

\proof{Proof sketch of Proposition~\ref{prop:wholelinucb}.}
We interpret Shared LinUCB as a generalization of Disjoint LinUCB, operating in a higher-dimensional feature space. The cumulative regret is defined as
\[
R(T) = \sum_{i=1}^N \sum_{t=1}^T r_{t, a_{i,t}^*} - \sum_{i=1}^N \sum_{t=1}^T r_{t, a_{i,t}}.
\]
Using the standard UCB analysis and applying the prediction error bound from Proposition~\ref{prop:disjoint1}, the regret can be bounded by
\[
R(T) \leq (1 + \eta) \sum_{t=1}^T \sqrt{ \left[ \sum_{i=1}^N \mathbf{x}_{i,c,t}, \mathbf{x}_{1,s,t}, \ldots, \mathbf{x}_{N,s,t} \right]^\top W_t^{-1} \left[ \sum_{i=1}^N \mathbf{x}_{i,c,t}, \mathbf{x}_{1,s,t}, \ldots, \mathbf{x}_{N,s,t} \right] }.
\]
Applying the log-determinant bound from Proposition~\ref{prop:disjoint3}, we obtain the final regret bound:
\[
R(T) = O\left( \left(d_c + \sum_{i=1}^{N} d_{i,s}\right) \sqrt{T \log T} \right).
\]
\QED \endproof

\begin{remark}
Shared LinUCB leverages shared and local features via a block-sparse context representation, reducing regret. However, its enlarged parameter space leads to high computational cost, motivating the development of more scalable alternatives such as NetLinUCB and SGD-UCB.
\end{remark}

\section{Supporting Results for NetLinUCB}\label{app:netlinucb}

In this section, we derive the weighted confidence radius used in NetLinUCB, building on the core ideas from Disjoint LinUCB. We begin by analyzing the error bound for a specific arm and node, focusing on the shared component of the model. We then generalize the result to include all nodes and arms, ultimately leading to a cumulative regret bound for the entire algorithm.

To motivate the structure of the weighted confidence radius, we adopt the notation and analytical approach of Disjoint LinUCB, adapting it to the multi-node setting with shared parameters; in particular, we focus on analyzing the shared component.

\begin{proposition}[Common component error bound]\label{prop:netlinucb1} At each time step $t$, for a given arm, we estimate the parameters and construct a confidence radius based on the accumulated historical information, namely the contexts \( \{\mathbf{x}_{\tau}\}_{\tau \in \Psi_t} \) and observed reward \( \{\mathbb{R}_{\tau}\}_{\tau \in \Psi_t} \). Focusing on the common context component, denoted by \( \mathbf{x}_{c, t} = [\mathbf{x}^0_{c, t}, \mathbf{0}] \), we obtain the following bound \(
\left| \mathbf{x}_{c, t}^\top \hat{\theta}^a - \mathbf{x}_{c, t}^\top \theta^a \right| \leq (\eta + 1) s_t\), which holds with probability at least \( 1 - \frac{1}{T} \), where \( \hat{\theta}^a = W_t^{-1} b_t \) and \( s_t = \sqrt{\mathbf{x}_{c, t}^\top W_t^{-1} \mathbf{x}_{c, t}} \).
\end{proposition}

\proof{Proof of Proposition~\ref{prop:netlinucb1}.}\label{prop:netlinucb1-proof}
Using notation consistent with Disjoint LinUCB, we have
\begin{align*}
\mathbf{x}_{c, t}^\top\hat{\theta}^a -  \mathbf{x}_{c, t}^\top \theta^a &= \mathbf{x}_{c, t}^\top W_{t}^{-1} b_t - \mathbf{x}_{c, t}^\top W_{t}^{-1} \left(I_d + D_t^\top D_t \right) \theta^a \\
&= \mathbf{x}_{c, t}^\top W_{t}^{-1} D_t^\top r_t - \mathbf{x}_{c, t}^\top W_{t}^{-1} \left( \theta^a + D_t^\top D_t \theta^a \right)\\
&= \mathbf{x}_{c, t}^\top W_{t}^{-1} D_t^\top (r_t - D_t \theta^a) - \mathbf{x}_{c, t}^\top W_{t}^{-1} \theta^a.
\end{align*}
Because of \( \| \theta^a \| \leq 1 \) under Assumption~\ref{assump:bounded}, we can bound:
\[
\left| \mathbf{x}_{c, t}^\top\hat{\theta}^a -  \mathbf{x}_{c, t}^\top \theta^a \right| \leq \left| \mathbf{x}_{c, t}^\top W_{t}^{-1} D_t^\top (r_t - D_t \theta^a) \right| + \left\| W_{t}^{-1} \mathbf{x}_{c, t} \right\|.
\]
The right-hand side above decomposes the prediction error into two components: a variance term and a bias term. We analyze these two terms separately. 

Given the statistical independence of samples indexed in \( \Psi_t \) and bounded rewards, we apply Hoeffding's inequality to the variance term,
\[
\Pr\left( \left| \mathbf{x}_{c, t}^\top W_{t}^{-1} D_t^\top (r_t - D_t \theta^a) \right| > \eta s_t \right) \leq 2 \exp \left( -\frac{2 \eta^2 s^2_t}{ \left\| D_t W_{t}^{-1} \mathbf{x}_{c, t} \right\|^2} \right)\leq 2 \exp(-2 \eta^2) =  \frac{1}{T},
\]
where the last inequality used the fact that
\[
s^2_t = \mathbf{x}_{c, t}^\top W_{t}^{-1} \mathbf{x}_{c, t} = \mathbf{x}_{c, t}^\top W_{t}^{-1} (I_d + D_t^\top D_t) W_{t}^{-1} \mathbf{x}_{c, t} 
\ge \mathbf{x}_{c, t}^\top W_{t}^{-1} D_t^\top D_t W_{t}^{-1} \mathbf{x}_{c, t} = \left( D_t W_{t}^{-1} \mathbf{x}_{c, t} \right)^2.
\]

Now applying a union bound, we can guarantee, with probability at least \( 1 - 1 / T \), that for all arms \( a \in \mathcal{A} \),
\[
\left| \mathbf{x}_{c, t}^\top W_{t}^{-1} D_t^\top (r_t - D_t \theta^a) \right| \leq \eta s_t.
\]

We next bound the bias term in the Equation above:
\[
\left\| W_{t}^{-1} \mathbf{x}_{c, t} \right\|= \sqrt{\mathbf{x}_{c, t}^\top W_{t}^{-1} I_d W_{t}^{-1} \mathbf{x}_{c, t}} \leq \sqrt{\mathbf{x}_{c, t}^\top W_{t}^{-1} (I_d + D_t^\top D_t) W_{t}^{-1} \mathbf{x}_{c, t}} \leq \sqrt{\mathbf{x}_{c, t}^\top W_t^{-1} \mathbf{x}_{c, t}} = s_t.
\]

This concludes the proof by showing that the total error is bounded by \((\eta + 1)s_t\) with high probability.
\QED \endproof

Having selected the most appropriate form of the confidence radius, the proof below completes the argument in Proposition~\ref{prop:netlinucb-bound} of Section~\ref{sec:netlinucb}, providing a detailed explanation of the weighted error bound.

\proof{Proof of Proposition~\ref{prop:netlinucb-bound}.}\label{prop:netlinucb-bound-proof}
Using the notation from the Disjoint LinUCB algorithm, and noting that the shared parameter $\theta_c$ is consistent across all nodes, we proceed as follows.
\begin{align*}
\mathbf{x}_{ c, t}^\top\tilde{\theta}_t -  \mathbf{x}_{ c, t}^\top \theta &= \sum_{j=1}^N \omega_j \mathbf{x}_{ c, t}^\top \tilde{\theta}_{j, t}  - \sum_{j=1}^N \omega_j \mathbf{x}_{ c, t}^\top \theta_j\\
&= \sum_{j=1}^N \omega_j \mathbf{x}_{ c, t}^\top W_{j, t}^{-1} b_{j, t} - \sum_{j=1}^N \omega_j \mathbf{x}_{ c, t}^\top W_{j, t}^{-1} \left(I_d + D_{j, t}^\top D_{j, t} \right) \theta_j\\
&= \sum_{j=1}^N \omega_j  \mathbf{x}_{ c, t}^\top W_{j, t}^{-1} D_{j, t}^\top r_{j, t} - \sum_{j=1}^N \omega_j \mathbf{x}_{ c, t}^\top W_{j, t}^{-1} \left( \theta_j + D_{j, t}^\top D_{j, t} \theta_j \right)\\
&= \sum_{j=1}^N (\omega_j \mathbf{x}_{ c, t})^\top W_{j, t}^{-1} D_{j, t}^\top (r_{j, t} - D_{j, t} \theta_j) - \sum_{j=1}^N (\omega_j \mathbf{x}_{ c, t})^\top W_{j, t}^{-1} \theta_j.
\end{align*}
Because of \(\| \theta \| \leq 1 \) under Assumption~\ref{assump:bounded}, it follows that
\[
\left| \mathbf{x}_{ c, t}^\top\tilde{\theta}_{ t} -  \mathbf{x}_{ c, t}^\top \theta \right| \leq  \left| \sum_{j=1}^N (\omega_j \mathbf{x}_{ c, t})^\top W_{j, t}^{-1} D_{j, t}^\top (r_{j, t} - D_{j, t} \theta_j) \right| + \left\| \sum_{j=1}^N (\omega_j \mathbf{x}_{ c, t})^\top W_{j, t}^{-1} \right\|.
\]
This decomposition splits the prediction error into a variance term (first) and a bias term (second).

\textbf{Variance term.} Define the martingale difference sequence for each node \( j \) as \(M_j = \omega_{j} \, \mathbf{x}_{c,t}^\top W_{j,t}^{-1} D_{j,t}^\top (r_{j,t} - D_{j,t} \theta)\). Owing to the statistical independence of samples indexed by \( \Psi_t \), we have \(E[r_{j, t} - D_{j, t} \theta_j] = \mathbb{E}[M_j] = 0\) for all nodes \(j\). By the Cauchy-Schwarz inequality, \( M_j \) is bounded by
\[
|M_j| \leq \omega_{j} \left\| D_{j,t} W_{j,t}^{-1} \mathbf{x}_{c,t} \right\| \cdot \left\| r_{j,t} - D_{j,t} \theta \right\|.
\]
Define \(s_{j,t} = \sqrt{\mathbf{x}_{c,t}^\top W_{j,t}^{-1} \mathbf{x}_{c,t}}\) for each node \(j\). Applying Hoeffding's inequality to the variance term,
\begin{align*}
\Pr\left( \left| \mathbf{x}_{ c, t}^\top W_{j, t}^{-1} D_{j, t}^\top (r_{j, t} - D_{j, t} \theta_j) \right| > \eta s_{j,t} \right) &\leq \exp \left( -\frac{2 \eta^2 s^2_{j, t}}{ \left\| D_{j, t} W_{j, t}^{-1} \mathbf{x}_{ c, t} \right\|^2} \right)\\
 &\leq  \exp(-2 \eta^2) =  (\frac{\delta}{2T})^4,
\end{align*}
where the first inequality follows from the fact that
\begin{align*}
s^2_{j, t} &= \mathbf{x}_{ c, t}^\top W_{j, t}^{-1} \mathbf{x}_{ c, t} =  \mathbf{x}_{ c, t}^\top W_{j, t}^{-1} (I_d + D_{j, t}^\top D_{j, t}) W_{j, t}^{-1} \mathbf{x}_{ c, t} \\
&\ge  \mathbf{x}_{ c, t}^\top W_{j, t}^{-1} D_{j, t}^\top D_{j, t} W_{j, t}^{-1} \mathbf{x}_{ c, t} = \left( D_{j, t} W_{j, t}^{-1} \mathbf{x}_{ c, t} \right)^2.
\end{align*}

According to Azuma's inequality, for a martingale sequence \( \{S_n\} \) with bounded increments \( |S_k - S_{k-1}| \leq c_k \) implies
\[
\Pr\left( \left| S_N - S_0 \right| \geq \lambda \right) \leq 2 \exp\left( -\frac{\lambda^2}{2 \sum_{k=1}^N c_k^2} \right).
\]
Substituting \( \lambda = \eta \cdot \tilde{s}_{t} \) and \( c_j = \omega_{j} \eta \cdot s_{j,t} \), we obtain
\[
\sum_{j=1}^N c_j^2 =\sum_{j=1}^N (\omega_{j}^a)^2 \eta^2 s_{j,t}^2 = 
\sum_{j=1}^N (\omega_{j}^a)^2 \eta^2  \mathbf{x}_{ c, t}^\top W_{j, t}^{-1} \mathbf{x}_{ c, t} = \sum_{j=1}^N (\omega_{j}^a \mathbf{x}_{ c, t})^\top W_{j, t}^{-1} (\omega_{j}^a \mathbf{x}_{ c, t}).
\]
Hence, the inequality $\Pr\left( \left| \sum_{j=1}^N M_j \right| > \eta \tilde{s}_{t} \right) \leq 2 \exp\left( -\frac{\lambda^2}{2 \sum_{k=1}^N c_k^2} \right)$ becomes
\begin{align*}
\Pr\left( \left| \sum_{j=1}^N M_j \right| > \eta \tilde{s}_{t} \right) &\leq 2 \exp\left( -\frac{\eta^2 (\tilde{s}_{t})^2}{2 \sum_{j=1}^N (\omega_{j}^a \mathbf{x}_{ c, t})^\top W_{j, t}^{-1} (\omega_{j}^a \mathbf{x}_{ c, t})} \right)\\
& \leq 2 \exp\left(-\frac{\eta^2}{2} \right) = \frac{1}{T}.
\end{align*}
Applying a union bound over all arms \( a \in \mathcal{A} \), we conclude that, with probability at least \( 1 - 1 / T \),
\[
\sum_{j=1}^N (\omega_j \mathbf{x}_{ c, t})^\top W_{j, t}^{-1} D_{j, t}^\top (r_{j, t} - D_{j, t} \theta_j) \leq \eta \tilde{s}_{t}.
\]

\textbf{Bias term.} 
By expanding the squared $\ell_2$-norm of the bias term, we obtain
\[
\left\| \sum_{j=1}^N (\omega_j \mathbf{x}_{c,t})^\top W_{j,t}^{-1} \right\|^2 = \sum_{j=1}^N \sum_{q=1}^N (\omega_j \mathbf{x}_{c,t})^\top W_{j,t}^{-1} W_{q,t}^{-1} (\omega_{q} \mathbf{x}_{c,t}).
\]
Define the threshold $\kappa_t := \frac{1}{1 + \max_j \lambda_{\text{max}}(D_{j,t}^\top D_{j,t})}$. All nodes in the network $\mathcal{G}$ are compared with the threshold. Let $\mathcal{U}_t = \{ j \mid \omega_{j} \geq \kappa_t \}$ denote nodes with large weights, and $\mathcal{Z}_t = \{ q \mid \omega_{q} < \kappa_t \}$ denote nodes with small weights. Define $n_1 := |\mathcal{U}_t|$ and $n_2 := |\mathcal{Z}_t|$. It follows that $n_1 + n_2 = N$. The double sum decomposes into three parts:
\[
\left\| \sum_{j=1}^N (\cdot) \right\|^2 = \underbrace{\sum_{j,q \in \mathcal{U}} (\cdot)}_{\text{Term 1}} + \underbrace{\sum_{j,q \in \mathcal{Z}} (\cdot)}_{\text{Term 2}} + \underbrace{2\sum_{j \in \mathcal{U}, q \in \mathcal{Z}} (\cdot)}_{\text{Term 3}}.
\]

For nodes $j$ where \(\omega_{j} \geq \kappa_t \), we apply the matrix inequality \(I_d \preceq \omega_{j} (I_d + D_{q,t}^\top D_{q,t})\) and obtain
\[
\text{Term 1} \leq \sum_{j \in \mathcal{U}} \sum_{q \in \mathcal{U}} (\omega_{j} \mathbf{x}_{c,t})^\top W_{j,t}^{-1} \omega_{j} (I_d + D_{q,t}^\top D_{q,t}) W_{q,t}^{-1} (\omega_{q} \mathbf{x}_{c,t}).
\]
Owing to the symmetry and diagonal dominance of (\( I_d + D_{q,t}^\top D_{q,t} \)), we retain for Term 1 the original bound corresponding to Disjoint LinUCB,
\[
\text{Term 1} \leq \sum_{j=1}^N (\omega_j \mathbf{x}_{c,t})^\top W_{j,t}^{-1} (\omega_j \mathbf{x}_{c,t}) = \tilde{s}_{t}^2.
\]

Because  $\sum_{j \in\mathcal{U}} \omega_{j} \leq n_1$, $\sum_{q \in \mathcal{Z}} \omega_{q} \leq 1 - \kappa_t n_1$, and design matrix $W_{t} = I + \sum_{\tau=1}^t \mathbf{x}_{\tau} \mathbf{x}_{\tau}^\top \succeq I$. For Terms 2 and 3, we establish the following key inequality:
\begin{align*}
|\text{Term 2} + \text{Term 3}| &\leq R^2 \left[ 2\left(\sum_{j \in\mathcal{U}} \omega_{j}\right) \left(\sum_{\text{small } q} \omega_{q}\right) + \left(\sum_{q \in \mathcal{Z}} \omega_{q}\right)^2 \right] \\
&\leq R^2 \left[ 2(1 - \kappa_t n_1) n_1 + (1 - \kappa_t n_1)^2 \right],
\end{align*}
where $R := \max_{i \in [N],\, t \in [T]} \| \mathbf{x}_{i,c,t} \| < \infty$.

Define $f(n_1) := 2(1 - \kappa_t n_1) n_1 + (1 - \kappa_t n_1)^2$, with domain $0 \leq n_1 \leq \lfloor \frac{1}{\kappa_t} \rfloor$, because the condition \(1 - \kappa_t n_1 \geq 0\) is necessary for the existence of \(\mathcal{Z}\). 

To find the maximum of $f(n_1)$, we compute the derivative \(
f'(n_1) = 2(1 - \kappa_t n_1) + 2n_1(-\kappa_t) + 2(1 - \kappa_t n_1)(-\kappa_t) = 2 - 4\kappa_t n_1 - 2\kappa_t + 2\kappa_t^2 n_1\). Setting it to zero yields
\[
(2\kappa_t^2 - 4\kappa_t) n_1 + (2 - 2\kappa_t) = 0 \implies n_1^* = \frac{1 - \kappa_t}{\kappa_t(2 - \kappa_t)}.
\]
We know \(
\frac{1 - \kappa_t}{\kappa_t(2 - \kappa_t)} \leq \frac{1}{\kappa_t} \iff 1 - \kappa_t \leq 2 - \kappa_t \), so the maximum value in the domain is
\[
f(n_1^*) = \frac{1}{\kappa_t(2 - \kappa_t)}.
\]

Therefore, the upper bound for the latter two terms, 
\[
|\text{Term 2} + \text{Term 3}| \leq \frac{R^2}{\kappa_t(2 - \kappa_t)}.
\]

We now consider the behavior of the cross terms (Term~2 and Term~3) over different time scales.

1. \textit{For small $t$:} In the early stages, $\kappa_t$ is bounded away from zero, and these two terms remain uniformly bounded as $0 \leq n_1 \leq \lfloor \frac{1}{\kappa_t} \rfloor$.  
Their cumulative contribution to regret from time step \(1\) to \(t\) is therefore bounded by \(O(\log t)\), which is negligible compared to the dominant 
\(O(\sqrt{t})\) term.

2. \textit{For large $t$:}  
As $t$ grows, $\lambda_{\max}(D_t^\top D_t) \to \infty$, $\kappa_t \to 0$ and all weights satisfy $\omega_j > \kappa_t$.  
Consequently, $n_2 \to 0$, and both Term~2 and Term~3 vanish.

Overall, the total contribution of the cross terms over all $t$ is bounded:
\[
\sum_{t=1}^T |\text{Term 2} + \text{Term 3}| = O(1),
\]
which is asymptotically negligible compared to the dominant $\tilde{O}(\sqrt{T})$ regret term.  
Hence, the prediction error satisfies
\[
\left| \mathbf{x}_{i,c,t}^\top \tilde{\theta}_{i,t}^a - \mathbf{x}_{i,c,t}^\top \theta_i^a \right|
\leq (\eta + 1)\,\tilde{s}_t + \zeta_t,
\]
where $\tilde{s}_{t} = \sqrt{\sum_{j=1}^N (\omega_{j} \mathbf{x}_{c,t})^\top W_{j,t}^{-1} (\omega_{j} \mathbf{x}_{c,t})} $, $\zeta_t$ collects the cross-term contributions with $\sum_{t=1}^T \zeta_t = O(1)$.

\QED \endproof

\begin{remark}\label{remark:tail_prob}
The quantity \(\eta = \sqrt{\log(2T)/2}\) serves as a confidence parameter that governs the tail probability in the associated concentration bound. It is derived from Hoeffding-type inequalities and ensures that the prediction error remains bounded with high probability \(1 - 1/T\) over \(T\) time steps. In the NetLinUCB algorithm, the exploration parameter is set as \(\alpha^{\text{ridge}} \leq 1 + \eta\), which directly controls the width of the upper confidence bound.
\end{remark}

With the error bound established, we now derive the regret bound of the algorithm over all time steps and all nodes. The following proof completes the argument for Theorem~\ref{thm:netlinucb-regret} in Section~\ref{sec:netlinucb}.

\proof{Proof of Theorem~\ref{thm:netlinucb-regret}.}\label{thm:netlinucb-proof} We decompose the total regret into common and node-specific component.

\textbf{Common component.} Based on the results established in previous propositions, we bound the cumulative regret incurred by the shared component as
\begin{align*}
R_1(T) &= \sum_{t=1}^T r_{c, t}^{a_{t}^*} - \sum_{t=1}^T r_{c, t}^{a_{t}} \leq (1+\eta) \sum_{i=1}^{N} \sum_{t=1}^T \sqrt{\sum_{j=1}^{N} (\omega_{ji}^{a_t}\mathbf{x}_{i, c, t})^\top (W^{a^t}_{j, t})^{-1} (\omega_{ji}^{a_t}\mathbf{x}_{i, c, t})} + (1+\eta) \sum_{i=1}^N \sum_{t=1}^T \zeta_{i,t},
\end{align*}
where $\sum_{t=1}^T \zeta_{i,t} = O(1)$ for each node $i$, representing the bounded contributions of bias cross-terms. Because these grow $O(\log T)$ with time, they are asymptotically negligible compared to the dominant $O(\sqrt{T})$ term and are henceforth omitted from the analysis.

For each node \( j \), the regularized covariance matrix \( W_{j,t}^k \) associated with arm $a^{(k)}$ is defined as \( W_{j,t}^k = I + \sum_{\tau\in\Psi_{j,t}^k} \mathbf{x}_{j,\tau} \mathbf{x}_{j,\tau}^\top\), where \( \Psi_{j,t}^k \) denotes the set of time steps up to $t$ at which node \( j \) selected arm \( a^{(k)} \), and \( n_{j,t}^k = |\Psi_{j,t}^k| \) is the corresponding count.

We decompose \( W_{j,t}^k \) into two parts, \(W_{j,t}^k = \underbrace{I}_{A} + \underbrace{\sum_{\tau\in\Psi_{j,t}^k} \mathbf{x}_{j,\tau} \mathbf{x}_{j,\tau}^\top}_{B}\). By the eigenvalue property of positive semi-definite matrices,
\[
\lambda_{\min}(A + B) \geq \lambda_{\min}(A) + \lambda_{\min}(B).
\]
Given that \( A = I \), we have \( \lambda_{\min}(A) = 1 \). In the LinUCB setting, the context vectors are bounded, i.i.d., and span the feature space, which implies that the population covariance matrix $\mathbb{E}[\mathbf{x}^\top \mathbf{x}]$ is full-rank with $\lambda_{\min} > 0$. Consequently, there exists a positive definite matrix \( \mathcal{V} \), s.t. \( \frac{1}{n_{j,t}^k} \sum_{\tau\in\Psi_{j,t}^k} \mathbf{x}_{j,\tau} \mathbf{x}_{j,\tau}^\top \succ \mathcal{V} \). Consequently, the matrix \( B = \sum_{\tau\in\Psi_{j,t}^k} \mathbf{x}_{j,\tau} \mathbf{x}_{j,\tau}^\top \) satisfies \(
\lambda_{min}(B) \geq n_{j,t}^k \cdot\lambda_{min}(\mathcal{V})
\). Let \( c := \lambda_{min}(\mathcal{V}) > 0 \), then \( \lambda_{\min}(B) \geq c \cdot n_{j,t}^k \). Hence, \( \lambda_{\min}(W_{j,t}) \geq 1 + c \cdot n_{j,t}^k \); for any vector \( \mathbf{x}_{i,c,t} \) with \( \|\mathbf{x}_{i,c,t}\| \leq 1 \),
\[
\mathbf{x}_{i,c,t}^\top W_{j,t}^{-1} \mathbf{x}_{i,c,t} \leq \frac{\|\mathbf{x}_{i,c,t}\|^2}{\lambda_{\min}(W_{j,t})} \leq \frac{d_c}{1 + c \cdot n_{j,t}^k}.
\]

Let $n_t^k = \sum_{j=1}^N |\Psi_{j,t}^k|$ denote the total number of times arm $k$ has been selected across all nodes up to time $t$. The weights \(\omega_{ji}^k \), which quantify the influence of node $j$ on node $i$ for arm $k$, are updated via \(
\omega_{ji}^k \propto {\frac{n_{i,t}^k \cdot n_{j,t}^k}{(n_t^k)^2}} + {\frac{1}{\text{dist}(\mathbf{x}_{i,c,t}, \mathbf{x}_{j,c,t})}}\) followed by symmetric normalization and smoothing to ensure \(
\sum_{j=1}^{N} \omega_{ji}^k = 1\) for eah \(i\), \(k\). Although these weight updates \( \omega_{ji}^k \) depend on both the arm-selection term and the context-similarity term, our analysis focuses on the arm-selection component because it directly impacts the exploration counts \( n_{j,t}^k \) that appear in the regret bound, while the similarity-based alignment term only affects the normalization and does not alter the leading-order regret dependence.

The instantaneous regret \( \Delta_{i,c,t} = r_{i, c, t}^{a_{t}^*} - r_{i, c, t}^{a_{t}} \) for node $i$ and the selected arm $k=a_t$ is bounded as follows:
\begin{align*}
\Delta_{i, c,t} &\leq (1+\eta) \sqrt{\sum_{j=1}^{N} (\omega_{ji}^k)^2 \frac{d_c}{1 + c \cdot n_{j,t}^k}} = (1+\eta) \sqrt{\sum_{j=1}^{N} \left(\frac{\frac{n_{i,t}^k \cdot n_{j,t}^k}{(n_t^k)^2}}{\frac{n_{i,t}^k \cdot \sum_j n_{j,t}^k}{(n_t^k)^2} }\right)^2 \frac{d_c}{1 + c \cdot n_{j,t}^k}} \\
&= (1+\eta) \sqrt{\sum_{j=1}^{N} \left(\frac{n_{j,t}^k}{n_t^k}\right)^2\frac{d_c}{1 + c \cdot n_{j,t}^k}} \leq (1+\eta)\sqrt{\frac{d_c}{(n_t^k)^2} \sum_{j=1}^{N} \frac{(n_{j,t}^k)^2}{c \cdot n_{j,t}^k}}\\
&= (1+\eta)\sqrt{\frac{d_c}{c(n_t^k)^2} \sum_{j=1}^{N} n_{j,t}^k} = (1+\eta)\sqrt{\frac{d_c}{c(n_t^k)^2} \sum_{j=1}^{N} n_{j,t}^k} = (1+\eta)\sqrt{\frac{d_c}{cn_t^k}}.
\end{align*}

Because each regret term of the form $\sqrt{\frac{1}{n_t^k}}$ is only incurred when arm $k$ is selected. That is, across all nodes, for each arm we accumulate the terms of \(\sqrt{\frac{d_c}{cn_t^k}}\) from $n_t^k=1$ to $n_t^k=n_T^k$, and at time step $T$, \(\sum_{k\in\mathcal{A}} n_T^k = NT\).

Therefore, for the final regret, we rewrite the total regret in terms of arm selections,
$$
R_1(T)=\sum_{i=1}^N \sum_{t=1}^T \Delta_{i,c,t} = \sum_{k \in\mathcal{A}} \sum_{m=1}^{n_T^k} (1+\eta)\sqrt{\frac{d_c}{c\cdot m}}.
$$

Applying the Cauchy–Schwarz inequality to the sum over $m$, we have,
$$
R_1(T)= \sum_{k \in\mathcal{A}} \sum_{m=1}^{n_T^k} (1+\eta)\sqrt{\frac{d_c}{c\cdot m}}
\leq \sum_{k \in\mathcal{A}} (1+\eta)\sqrt{n_T^k\sum_{m=1}^{n_T^k}\frac{d_c}{c\cdot m}}.
$$
Using standard bandit analysis techniques and an integral approximation, we obtain
$$
\sqrt{n_T^k\sum_{m=1}^{n_T^k}\frac{d_c}{c\cdot m}} \leq \sqrt{\frac{n_T^k \cdot d_c}{c}\log(\frac{c\cdot n_T^k}{d_c})}.
$$
Then summing over all arms,
\begin{align*}
R_1(T) &\leq 2(1+\eta)\sum_{k \in\mathcal{A}}\sqrt{\frac{n_T^k \cdot d_c}{c}\log(\frac{c\cdot n_T^k}{d_c})} \leq 2(1+\eta)
\sqrt{(\sum_{k \in\mathcal{A}}\frac{n_T^k \cdot d_c}{c})(\sum_{k \in\mathcal{A}}\log(\frac{c\cdot n_T^k}{d_c}))}\\
&\leq 2(1+\eta)\sqrt{\frac{TNd_c}{c} [K\log(\frac{c}{d_c}) + \log((TN)^K)]} \leq 2(1+\eta)\sqrt{\frac{TNd_c}{c} K \log(\frac{cTN}{d_c})},
\end{align*}
where $K$ is the total number of arms in the set $\mathcal{A}$.

The cumulative regret for the shared component satisfies
\[
R_1(T) = O\left(\sqrt{\frac{d_cT N}{c} \log (\frac{cTN}{d_c})} \right).
\]

\textbf{Node-specific component.} For each node $i$, the local estimation of the node-specific parameter $\theta_{i,s}$ is conducted independently using a standard Disjoint LinUCB model. Based on the regret analysis of Disjoint LinUCB (see Proposition~\ref{prop:disjoint4}), the cumulative regret from all node-specific components satisfies
\[
R_2(T) \leq O\left( \sum_{i=1}^{N} d_{i, s} \sqrt{T \log T} \right).
\]
This term aggregates the regret contributions from all node-specific features, and scales linearly in both the node-specific context dimension $d_{i,s}$ and the number of nodes $N$.

\textbf{Total regret.} Combining the bounds for the shared and node-specific components, we obtain the following overall regret bound:
\[
R(T) = R_1(T) + R_2(T) = O\left( \sqrt{\frac{d_cT N}{c} \log (\frac{cTN}{d_c})} + \sum_{i=1}^{N} d_{i, s} \sqrt{T \log T} \right),
\]
where $N$ is the number of nodes in the network; $d_c$ and $d_{i,s}$ denote the dimensions of the common and node-specific context features for node $i$, respectively; and $c > 0$ is a constant that lower-bounds the minimum eigenvalue of the normalized context covariance matrix. This constant $c$ captures the diversity of the shared contexts and is implicitly related to the common feature dimension $d_c$; a larger value of $c$ implies more effective exploration.
\QED \endproof

\section{Supporting Results for Net-SGD-UCB}\label{app:net-sgd-ucb}

This section presents the detailed theoretical analysis of the Net-SGD-UCB algorithm. We begin with the proof to a key technical lemma (Lemma~\ref{lemma:log-smooth-harmonic}) that provides a bound on a class of accumulated sequences. We then build upon this lemma to prove Proposition~\ref{prop:grad-bound}, which establishes an upper bound on the confidence radius used in the Net-SGD-UCB updates. This proposition explicitly characterizes how algorithmic parameters influence the cumulative uncertainty.

\proof{Proof of Lemma~\ref{lemma:log-smooth-harmonic}, \citet{duchi_adaptive_2011}, Lemma 4.}\label{lemma:log-smooth-harmonic-proof}
Because of \( A_t \geq A_{t-1} \), \( \log A_t - \log A_{t-1} = \log\left(1 + \frac{a_t}{A_{t-1}} \right) \geq (\frac{a_t}{A_t}) \). Summing over time, we get \( \sum_{t=2}^T (\frac{a_t}{A_t}) \leq \log A_T - \log A_1\). Adding the first term \( \frac{a_1}{A_1} \leq 1 \), we conclude that
\[
\sum_{t=1}^T (\frac{a_t}{A_t}) \leq \log\left(\frac{A_T}{A_1} \right) + 1 \leq \log\left(1 + \frac{A_T}{A_1} \right) + 1.
\]
\QED \endproof

\proof{Proof of Proposition~\ref{prop:grad-bound}.}\label{prop:grad-bound-proof}
We apply exponential moving average (EMA) smoothing with parameter $\gamma$ to update the gradient accumulation matrix $G_t$, where
\[
G_t = \gamma G_{t-1} + (1-\gamma)\,\mathrm{diag}\bigl((\nabla_\theta \mathcal{L}_t)^{\odot 2}\bigr).
\]
Accumulating along time steps, we obtain
\[
G_t = (1 - \gamma) \sum_{\tau=1}^{t} \gamma^{t-\tau} \mathrm{diag}\bigl((\nabla_\theta \mathcal{L}_\tau)^{\odot 2}\bigr) + \gamma^{t-1} I.\]
Focusing on each diagonal entry, the $h$‑th component is given by
\[
(G_t)_{hh} = (1 - \gamma) \sum_{\tau=1}^{t} \gamma^{t-\tau} \left(\frac{\partial \mathcal{L}_{\tau}}{\partial \theta_h}\right)^2 + \gamma^{t-1}.
\]
Because the loss function $\mathcal{L}_{t}(\theta)$ is the squared loss and the reward function is linear with noise $\epsilon$ and a fixed unknown parameter $\theta$, we have
\[
\frac{\partial \mathcal{L}_{\tau}}{\partial \theta_h} = - (r_\tau - \mathbf{x}_\tau^\top \theta)(\mathbf{x}_\tau)_h = - \epsilon_\tau (\mathbf{x}_\tau)_h,
\]
where $(G_t)_{hh}$ denotes the $h$‑th diagonal entry and $(\mathbf{x}_s)_h$ denotes the $h$‑th component of the vector. So $(G_t)_{hh} = (1 - \gamma) \sum_{\tau=1}^{t} \gamma^{t-\tau} \epsilon_\tau^2 (\mathbf{x}_\tau)_h^2 + \gamma^{t-1}$.

Next, consider the sum of $\mathbf{x}_t^\top G^{-1} \mathbf{x}_t$ over time:
\begin{align*}
\sum_{t=1}^T \mathbf{x}_t^\top G^{-1} \mathbf{x}_t &= \sum_{t=1}^T \sum_{h=1}^d \frac{(\mathbf{x}_t)_h^2}{(G_t)_{hh}}\\
&= \sum_{h=1}^d \sum_{t=1}^T \frac{(\mathbf{x}_t)_h^2}{(1 - \gamma) \sum_{\tau=1}^{t} \gamma^{t-\tau} \epsilon_\tau^2 (\mathbf{x}_\tau)_h^2 + \gamma^{t-1}} \\
& \leq \sum_{h=1}^d \sum_{t=1}^T \frac{(\mathbf{x}_t)_h^2}{(1 - \gamma) \sum_{\tau=1}^{t} \gamma^{t-\tau} \epsilon_\tau^2 (\mathbf{x}_\tau)_h^2}\\
& = \sum_{h=1}^d \sum_{t=1}^T \frac{1}{(1-\gamma) \epsilon_t^2} \cdot \frac{\epsilon_t^2 (\mathbf{x}_t)_h^2}{\sum_{\tau=1}^{t} \gamma^{t-\tau} \epsilon_\tau^2 (\mathbf{x}_\tau)_h^2} \\
& \leq \sum_{h=1}^d \sum_{t=1}^T \frac{1}{(1-\gamma)^2 \epsilon_t^2} \cdot \frac{\epsilon_t^2 (\mathbf{x}_t)_h^2}{\sum_{\tau=1}^{t} \epsilon_\tau^2 (\mathbf{x}_\tau)_h^2},
\end{align*}
where $d=d_c+d_s$ is the dimension of the contexts. The last inequality holds because in our algorithm $\gamma \in [0,1]$ is chosen close to $1$ to ensure more stable parameter updates.

By applying Lemma~\ref{lemma:log-smooth-harmonic}, which states \(\sum_{t=1}^T (\frac{a_t}{A_t}) \leq 2 \log\left(1 + \frac{A_T}{A_1}\right)\), and using the boundedness of the context $\mathbf{x}_s$, we obtain
$$\sum_{t=1}^T \frac{\epsilon_t^2 (\mathbf{x}_t)_h^2}{\sum_{\tau=1}^{t} \epsilon_\tau^2 (\mathbf{x}_\tau)_h^2} 
\leq 2 \log ( 1+ \sum_{\tau=1}^{T} \epsilon_\tau^2 (\mathbf{x}_\tau)_h^2) \leq 2\log (\sigma^2 T).$$

We then apply Jensen’s inequality to handle the term, $\sum_{h=1}^d \sum_{t=1}^T \frac{1}{(1-\gamma)^2 \epsilon_t^2} \cdot \frac{\epsilon_t^2 (\mathbf{x}_t)_h^2}{\sum_{\tau=1}^{t} \epsilon_\tau^2 (\mathbf{x}_\tau)_h^2}$. Because $\epsilon_t$ is independent of $\mathcal{F}_{t-1}$ and is a noise term with variance $\sigma^2$, we calculate the expectation of this term to facilitate analysis of the final result. We have
\begin{align*}
\mathbb{E}\left[ \frac{1}{\epsilon_t^2} \cdot w_t \Big| \mathcal{F}_{t-1} \right] 
\leq \frac{1}{\mathbb{E}[\epsilon_t^2 | \mathcal{F}_{t-1}]} \mathbb{E}[w_t | \mathcal{F}_{t-1}] = \frac{1}{\sigma^2} \mathbb{E}[w_t | \mathcal{F}_{t-1}],
\end{align*}
where $w_t = \frac{\epsilon_t^2 (\mathbf{x}_t)_h^2}{\sum_{\tau=1}^{t} \epsilon_\tau^2 (\mathbf{x}_\tau)_h^2}$. Therefore, we have
\begin{align*}
\mathbb{E}\left[\sum_{t=1}^T \frac{1}{(1-\gamma)^2 \epsilon_t^2} \cdot \frac{\epsilon_t^2 (\mathbf{x}_t)_h^2}{\sum_{\tau=1}^{t} \epsilon_\tau^2 (\mathbf{x}_\tau)_h^2} \right] 
&\leq \frac{1}{(1-\gamma)^2\sigma^2} \mathbb{E}\left[ \sum_{t=1}^T \frac{\epsilon_t^2 (\mathbf{x}_t)_h^2}{\sum_{\tau=1}^{t} \epsilon_\tau^2 (\mathbf{x}_\tau)_h^2} \right] \\
&\leq \frac{2\log (\sigma^2 T)}{(1-\gamma)^2\sigma^2}.
\end{align*}
Summing over $h = 1,\dots,d$ yields the final result.

Next, we analyze the variance introduced by momentum update, which is
\begin{align*}
v_{i,k}^{(t)} &= \mu v_{i,k}^{(t-1)} + (1-\mu)\nabla\mathcal{L}^{(t)} \\
&= (1-\mu)\sum_{\tau=1}^{t} \mu^{t-\tau-1} \nabla\mathcal{L}^{(\tau)}.
\end{align*}
Because the arm selection and reward at each time step are independent, the expectation becomes
\begin{align*}
\mathbb{E}\left[\|v_{i,k}^{(t)}\|^2\right] &= (1-\mu)^2 \mathbb{E}\left[\left\|\sum_{\tau=1}^{t} \mu^{t-\tau-1} \nabla\mathcal{L}^{(\tau)}\right\|^2\right] \\
&= (1-\mu)^2 \sum_{\tau=1}^{t} \mu^{2(t-\tau)} \mathbb{E}\left[\|\nabla\mathcal{L}^{(\tau)}\|^2\right].
\end{align*}
Geometric series summation satisfies\(\sum_{\tau=1}^{t} \mu^{2(t-\tau)} = \frac{1 - \mu^{2t}}{1 - \mu^2} \leq \frac{1}{1 - \mu^2}\), we obtain the temporal cumulative variance bound:
\begin{align*}
\sum_{t=1}^T \mathbb{E}\left[\|v_{i,k}^{(t)}\|^2\right] &\leq (1-\mu)^2 \sum_{t=1}^T \sum_{\tau=1}^{t} \mu^{2(t-\tau)}\mathbb{E}\left[\|\nabla\mathcal{L}^{(\tau)}\|^2\right] \\
&= (1-\mu)^2 \sum_{\tau=1}^{t} \mathbb{E}\left[\|\nabla\mathcal{L}^{(\tau)}\|^2\right] \sum_{t=\tau}^T \mu^{2(t-\tau)} \\
&\leq \frac{(1-\mu)^2}{1 - \mu^2} \sum_{\tau=1}^{t} \mathbb{E}\left[\|\nabla\mathcal{L}^{(\tau)}\|^2\right].
\end{align*}
According to Lemma 3 of~\citet{reddi_convergence_2019}, the cumulative effect of momentum is upper bounded by \(O(\log T)\). Then we obtain \(\sum_{\tau=1}^{t}\mathbb{E}\left[\|\nabla\mathcal{L}^{(\tau)}\|^2\right] \leq \sigma^2 L^2 \log T\). Thus,
\begin{align*}
\sum_{t=1}^T \mathbb{E}\left[\|v_{i,k}^{(t)}\|^2\right] &\leq \frac{(1-\mu)^2 \sigma^2 L^2 \log T}{1 - \mu^2} = \frac{(1-\mu) \sigma^2 L^2 \log T}{1 + \mu}.
\end{align*}
Combining both terms, we obtain
\[
 \mathbf{x}_t^\top G^{-1} \mathbf{x}_t \leq \sum_{t=1}^T \frac{2d\log (\sigma^2 T)}{(1-\gamma)^2\sigma^2}+ \frac{(1-\mu)\sigma^2\log T}{(1+\mu)(1-\gamma)},
\]
where $d$ is the dimension of the contexts, $\sigma^2$ is the variance of the noise of our reward function, $\mu$ is the momentum parameter, $\gamma \in [0,1]$ is the smoothing parameter.
\QED \endproof

Building upon the confidence radius bound derived in Proposition~\ref{prop:grad-bound}, we now analyze the cumulative regret incurred by the Net-SGD-UCB algorithm at each node. The complete proof of Proposition~\ref{prop:sgd-ucb-regret-bound} is provided below.

\proof{Proof of Proposition~\ref{prop:sgd-ucb-regret-bound}.}\label{prop:sgd-ucb-regret-bound-proof}
We first establish the high-probability condition required for the analysis. 
The algorithm maintains a global accumulation matrix $G_{t}$ for each node $i$, which is updated using the gradients of the selected arms. 
For arm $k$ at time $t$, the corresponding confidence interval is given by
\[
\hat{\mu}_{t} \pm \alpha^{\text{sgd}} \sqrt{\mathbf{x}_{t}^{k\top} G_{t}^{-1} \mathbf{x}_{t}}.
\]
We aim to ensure that this confidence bound holds with high probability.
This guarantee depends on the stochastic noise in the reward observations. The noise \(\epsilon_t\) in the reward propagates through the matrix \( G_t \), introducing randomness in the estimated gradients.
By controlling the cumulative effect of these noise terms over all coordinates and time steps, we ensure that the true mean reward lies within the stated confidence interval with probability at least $1-1/T$.

We define the normalized noise component for the $h$-th coordinate as
\[
z_{t,h} = \frac{(\mathbf{x}_{t})_h \epsilon_t}{(G_t)_{hh}},
\]
where $(G_t)_{hh}$ denotes the $h$-th diagonal entry, and $(\mathbf{x}_s)_h$ denotes the $h$-th entry of the context vector. 
Given that $\epsilon_t \sim \text{sub-Gaussian}(\sigma^2)$ and $\mathbf{x}_t$ is bounded, each $z_{t,h}$ is also sub-Gaussian. 
Applying the sub-Gaussian tail inequality, we obtain
\[
\mathbf{P}\left(\left|\sum_{t=1}^T z_{t,h}\right| \geq \beta \mid \mathcal{F}_{T-1}\right) \leq 2\exp\left(-\frac{a^2}{2\sigma^2}\right).
\]
By choosing $\beta = \sigma\sqrt{2\log(2dT)}$ and applying a union bound over $h=1,\dots,d$, we conclude that with probability at least $1- \frac{1}{T}$,
\[
\left|\sum_{h=1}^d \sum_{t=1}^T \frac{(\mathbf{x}_t)_h \epsilon_t}{\sqrt{(G_t)_{hh}}}\right| \leq \sigma\sqrt{2dT\log(2dT)}.
\]
This concentration result guarantees that, with high probability, the confidence interval
\[
\hat{\mu}_t \pm \alpha^{\text{sgd}} \sqrt{\mathbf{x}_t^\top G_t^{-1} \mathbf{x}_t}
\]
contains the true expected reward. Because the sub-Gaussian noise contribute is of the same order as the confidence radius. Here, the exploration parameter \( \alpha^{\text{sgd}} \) is chosen such that it scales with \( 1 + \sigma^2 \), thereby ensuring that the noise-induced deviation is absorbed into the radius and the bound remains valid. 

We now turn to the regret analysis. Under the above high-probability event, the instantaneous regret incurred from selecting suboptimal arm $k_t$ instead of the optimal arm $k^\star$ is bounded by
\[
\Delta_{t} \leq 2\alpha^{\text{sgd}} \sqrt{\mathbf{x}_{t}^{k_t\top} G_{t}^{-1} \mathbf{x}_{t}^{k_t}}.
\]
Summing over time and applying the Cauchy--Schwarz inequality, we get
\[
R(T) \leq 2\alpha^{\text{sgd}} \sqrt{T \sum_{t=1}^T \mathbf{x}_t^\top G^{-1} \mathbf{x}_t}.
\]

From Proposition~\ref{prop:grad-bound}, we know
\[
\sum_{t=1}^T \mathbf{x}_t^\top G^{-1} \mathbf{x}_t \leq \frac{2d\log (\sigma^2 T)}{(1-\gamma)^2\sigma^2} + \frac{(1-\mu)\sigma^2 \log T}{(1+\mu)(1-\gamma)}.
\]
Substituting this bound into the regret expression yields the final bound:
\[
R(T) \leq 2\alpha^{\text{sgd}}\left( \frac{\sqrt{2dT \log (\sigma^2 T)}}{(1 - \gamma)\sigma} + \sigma \sqrt{\frac{(1 - \mu)T\log T}{(1 + \mu)(1 - \gamma)}} \right),
\]
where $d$ is the dimension of the contexts, $\sigma^2$ is the variance of the noise in the reward function, $\mu$ is the momentum parameter, and $\gamma$ is the smoothing parameter.
\QED \endproof

After establishing the regret bound for each individual node, we now extend the analysis to the entire network by incorporating the weighted influence among nodes through the inter-node weight matrix. We provide the full proof of Theorem~\ref{thm:sgd-ucb-network-regret} in Section~\ref{sec:sgd-ucb} below.

\proof{Proof of Theorem~\ref{thm:sgd-ucb-network-regret}.}\label{thm:sgd-ucb-network-regret-proof}
We decompose the cumulative regret into two parts: the contribution fro the shared parameter, denoted \( R_1(T) \), and the node-specific regret, denoted \( R_2(T) \).

When a suboptimal arm $a_t$ is selected, the regret arising from the shared parameter $\theta_c^{a_t}$ is analyzed in four steps: decomposition of the estimation error, control of the variance, bias analysis, and evaluation of the weighted confidence bound.

Each node \(i\) maintains an estimate \(\hat\theta_{c,i}^{a_t} \in \mathbb{R}^{d_c}\) of the true shared parameter \(\theta_c^{a_t}\), and defines the estimation error as \(\Delta_{i,c,t} = \hat\theta_{c,i}^{a_t} - \theta_c^{a_t}\). At each time step \(t\), node \(i\) updates its estimate by aggregating its neighbors' estimates and performing a stochastic gradient step, which is
\[
\hat\theta_{c,i,t}^{a_t} = \sum_{j=1}^{N} \omega_{ji}^{a_t} \hat\theta_{c,j,t-1}^{a_t} + \eta^{\text{sgd}}\,\nabla \mathcal{L}_{i,t}.
\]
Subtracting $\theta_c^{a_t}$ yields the recursive form of the estimation error:
\[
\Delta_{i,c,t} = \sum_{j=1}^{N} \omega_{ji}^{a_t} \Delta_{j,c,t-1} + \eta^{\text{sgd}}\,\nabla \mathcal{L}_{i,t}.
\]
This recursion separates the aggregation dynamics from the stochastic gradient error. Decomposing the error gives
\[
\Delta_{i,c,t} = \underbrace{\left( \Delta_{i,c,t} - \mathbb{E}[\Delta_{i,c,t}] \right)}_{\text{Variance}} + \underbrace{\left( \mathbb{E}[\Delta_{i,c,t}] - \theta_c^{a_t} \right)}_{\text{Bias}}.
\]

\textbf{Variance term.} By applying EMA smoothing with factor \(\gamma\), and following results from decentralized SGD~\citep{lian_can_2017}, the variance term is bound as $O(\frac{\sigma^2}{(1-\gamma)m})$, where $m$ denotes the number of SGD updates. Combining this with global updates and arm selection structure, we obtain the bound
\[
\mathbb{E}\left[ \| \Delta_{i,c,t} - \mathbb{E}[\Delta_{i,c,t}] \|^2 \right] \leq \frac{\sigma^2}{(1-\gamma) n_t^k},
\]
where \(\sigma^2\) is the noise variance and \(n_{t}^k\) denotes selected counts before time $t$ among all nodes for arm $k$.

Because updates are only triggered when arm $k$ is selected, we split the sum accordingly. Following the same analytical approach as in NetLinUCB, we aggregate over all time steps and nodes:
\begin{align*}
\sum_{i=1}^N \sum_{t=1}^T \frac{\sigma^2}{(1-\gamma) n_t^k}
&\leq \sum_{k\in \mathcal{A}} \sum_{m=1}^{n_T^k} \frac{\sigma^2}{(1-\gamma) m} \leq \sum_{k\in \mathcal{A}} \frac{1-\gamma}{\sigma^2} \log(\frac{\sigma^2}{(1-\gamma) n_T^k})\\
& \leq \frac{1-\gamma}{\sigma^2} \cdot K \log(\frac{\sigma^2}{(1-\gamma) \sum_{k \in \mathcal{A}}n_T^k}) = \frac{(1-\gamma)K}{\sigma^2}  \log(\frac{\sigma^2}{(1-\gamma) NT}),
\end{align*}
which becomes negligible as $T \to \infty$.

\textbf{Bias term.} Because all nodes share the same true parameter \(\theta_c^{a_t}\), the expected estimation bias vanishes:
\[
\sum_{t=1}^T \| \mathbb{E}[\Delta_{i,c,t}] - \theta_c^{a_t} \|^2 \leq C^2 T \max_j \| \theta_c^{a_t} - \theta_{c,j}^* \|^2 = 0,
\]
where \(C = \mathcal{O}(1)\) is a bounded constant.

\textbf{Confidence radius term.} As $T$ increases, both the variance and bias components become negligible. We now focus on bounding the contribution from the weighted confidence radius. The regret contribution becomes
$$\sum_{i=1}^N \sum_{t=1}^T \sqrt{\sum_{j=1}^n (\omega_{ji})^2 \mathbf{x}_{c, t}^\top G_{c,t}^{-1} \mathbf{x}_{c,t}}.$$

According to Proposition~\ref{prop:grad-bound}, $\sum_{t=1}^T \mathbf{x}_{t}^\top G_{t}^{-1} \mathbf{x}_{t} \leq \frac{\sqrt{2d_{i,s} \log (\sigma^2 T)}}{(1 - \gamma)\sigma} + \sigma \sqrt{\frac{(1 - \mu)\log T}{(1 + \mu)(1 - \gamma)}}$, and noting that the dominant component in $\omega_{ji}$ primarily arises from arm selection statistics,
\begin{align*}
\sum_{i=1}^N \sum_{t=1}^T \sqrt{\sum_{j=1}^n (\omega_{ji})^2 \mathbf{x}_{i, c, t}^\top G_{j,c,t}^{-1} \mathbf{x}_{i, c,t}} &= \sum_{i=1}^N \sum_{t=1}^T \sqrt{\sum_{j=1}^n (\frac{n_{j,t}^k}{n_t^k})^2 \mathbf{x}_{i, c, t}^\top G_{j,c,t}^{-1} \mathbf{x}_{i,c,t}}.
\end{align*}

\textbf{Final bound for shared component \(R_1\).} Applying the Cauchy–Schwarz inequality, we derive
\begin{align*}
\sum_{t=1}^T \sum_{i=1}^N r_{i,c,t}
&\leq \sqrt{ NT \cdot \sum_{t=1}^T \sum_{i=1}^N \sum_{j=1}^N (\frac{n_{j,t}^k}{n_t^k})^2 \mathbf{x}_{i, c, t}^\top G_{j,c,t}^{-1} \mathbf{x}_{i, c,t}} \\
&= \sqrt{ NT \cdot \sum_{t=1}^T \sum_{j=1}^N  (\frac{n_{j,t}^k}{n_t^k})^2 (\sum_{i=1}^N \mathbf{x}_{i, c, t})^\top G_{j,c,t}^{-1} (\sum_{i=1}^N \mathbf{x}_{i, c, t})}\\
&\leq \sqrt{ NT \cdot \sum_{t=1}^T (\sum_{i=1}^N \mathbf{x}_{i, c, t})^\top G_{j,c,t}^{-1} (\sum_{i=1}^N \mathbf{x}_{i, c, t})}\\
& = \mathcal{O}\left( \frac{\sqrt{ N d_c T \log (\sigma^2 T})}{1 - \gamma} \right),
\end{align*}
where the weighting factor satisfies $\sum_{j=1}^N  (\frac{n_{j,t}^k}{n_t^k})^2 \leq \frac{(\sum_{j=1}^Nn_{j,t}^k) \cdot (\sum_{j=1}^N n_{j,t}^k)}{(n_t^k)^2} = 1$. The final inequality holds because each matrix $G_{j,c,t}$ aggregates information from all nodes through the corresponding context vectors $\sum_{i=1}^N \mathbf{x}_{i, c, t}$. Thus, the bound from Proposition~\ref{prop:grad-bound} applies. The residual contributions from inner products involving shared contexts are absorbed into the analysis of the node-specific component $R_2(T)$.

\textbf{Final bound for node-specific component \(R_2\).} We now bound the regret arising from node-specific parameters, corresponding to \(N\) independent instances of SGD-UCB. Using the regret bound from Proposition~\ref{prop:sgd-ucb-regret-bound}, each node contributes
\[
O\left( \frac{\sqrt{2d_{i,s}T \log (\sigma^2 T)}}{(1 - \gamma)\sigma} + \sigma \sqrt{\frac{(1 - \mu)T\log T}{(1 + \mu)(1 - \gamma)}} \right). 
\]
Summing over all nodes yields
\[
R_2(T) =  O\left(\sum_{i=1}^{N} \frac{\sqrt{2d_{i,s}T \log (\sigma^2 T)}}{(1 - \gamma)\sigma} + N \sigma \sqrt{\frac{(1 - \mu)T\log T}{(1 + \mu)(1 - \gamma)}} \right).
\]
This completes the proof by combining both the shared and node-specific regret components.
\QED \endproof

\section{Supplementary Numerical Experiment Results}\label{app:numerical}

We present detailed results for additional simulation instances in Figures~\ref{fig:exp2} and~\ref{fig:exp3to5}. These experiments, conducted with a total of $N=12$ cities, complement the analysis in the main text (Section~\ref{sec:numerical}) and explore a range of challenging conditions. Specifically, we vary: (i) the proportion of shared versus local context dimensions, (ii) the variance of the context distribution, (iii) the number of arms, and (iv) the reward gap between the optimal and second-best actions. Each setting emphasizes different aspects of model generalization and exploration efficiency, and different algorithms exhibit varying levels of performance depending on the contextual structure and reward pattern.

\begin{figure*}[ht]
\centering
\includegraphics[width=0.7\linewidth]{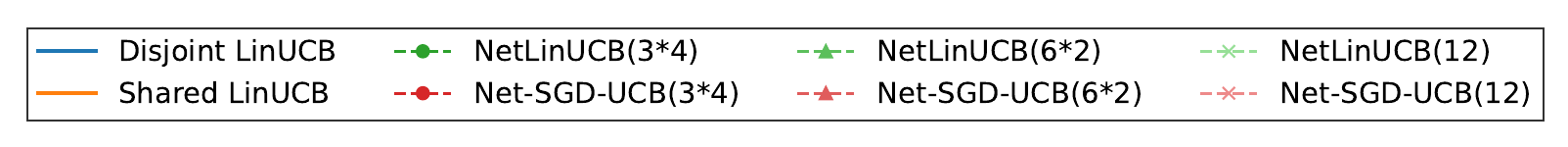}
\caption{Numerical legend. In all figures below, blue and orange denote Disjoint and Shared LinUCB, respectively. Red and green curves represent NetLinUCB and Net-SGD-UCB under varying simulation environment settings. Each label denotes a sub-network structure: \textbf{(3*4)} and \textbf{(6*2)} correspond to 3 or 2 fully connected sub-networks with 4 or 6 nodes each, respectively, while \textbf{(12)} denotes a single fully connected network. In Net-based models, information is shared within each sub-network but not across sub-networks. Disjoint LinUCB corresponds to the degenerate case \textbf{(1*12)} with no sharing, and Shared LinUCB to the extreme case \textbf{(12)} with full sharing.}
\end{figure*}

\begin{figure}[htb]
  \centering
  \begin{minipage}[b]{0.48\linewidth}
    \centering
    \includegraphics[width=\linewidth]{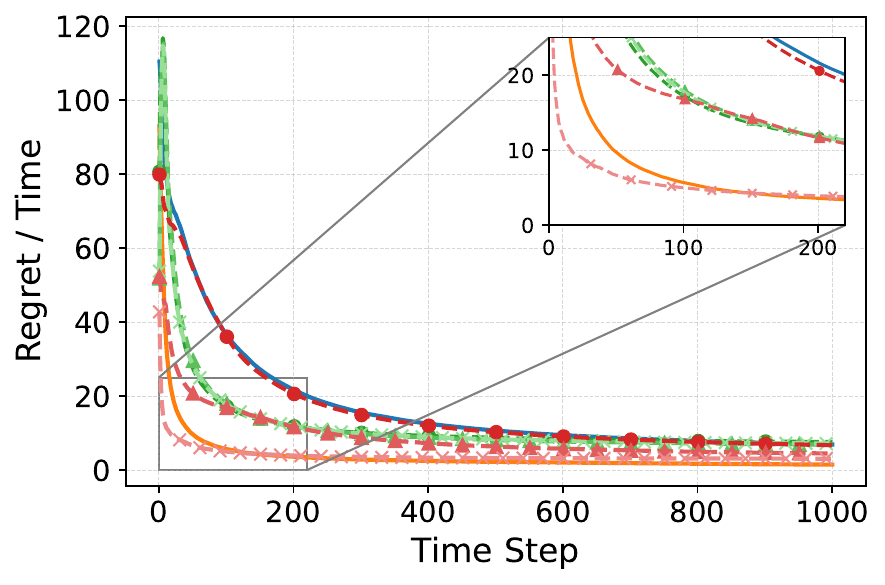}
    \caption*{(a) Low shared-to-specific context ratio, $(t, \tfrac{R(t)}{t})$}
  \end{minipage}
  \hfill
  \begin{minipage}[b]{0.48\linewidth}
    \centering
    \includegraphics[width=\linewidth]{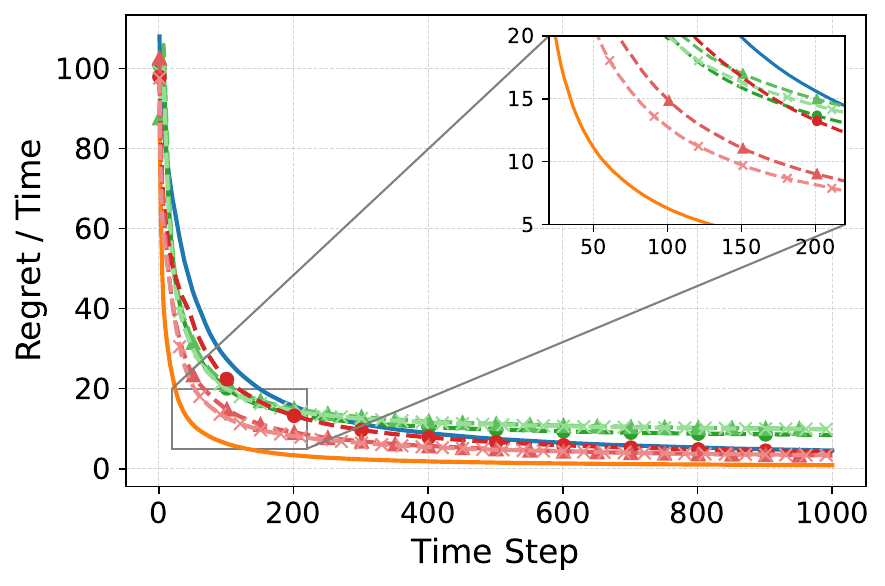}
    \caption*{(b) High shared-to-specific context ratio, $(t, \tfrac{R(t)}{t})$}
  \end{minipage}
  \caption{Performance under different shared-to-specific context ratios. When shared context dominates, Shared LinUCB outperforms others by leveraging transferable structure; in contrast, when local context dominates, NetLinUCB converges faster, and Disjoint LinUCB becomes more competitive due to its independence.}
  \label{fig:exp2}
\end{figure}

\begin{figure}[ht]
\centering

\begin{minipage}[c]{0.45\linewidth}
  \includegraphics[width=\linewidth]{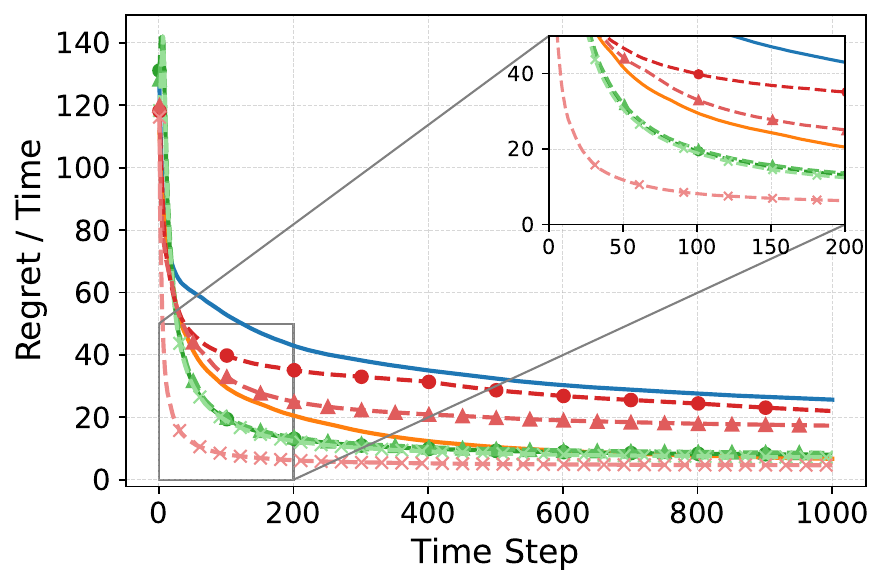}
  \centering
  \small (a) Default, $(t, \tfrac{R(t)}{t})$
\end{minipage}
\hspace{0.05\linewidth}
\begin{minipage}[c]{0.45\linewidth}
  \includegraphics[width=\linewidth]{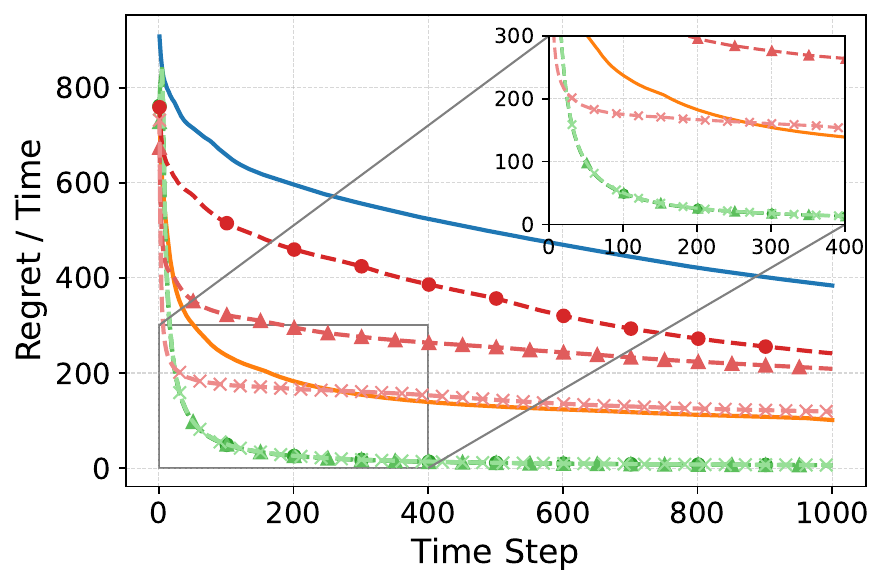}
  \centering
  \small (b) Large context outlier, $(t, \tfrac{R(t)}{t})$
\end{minipage}

\begin{minipage}[c]{0.45\linewidth}
  \includegraphics[width=\linewidth]{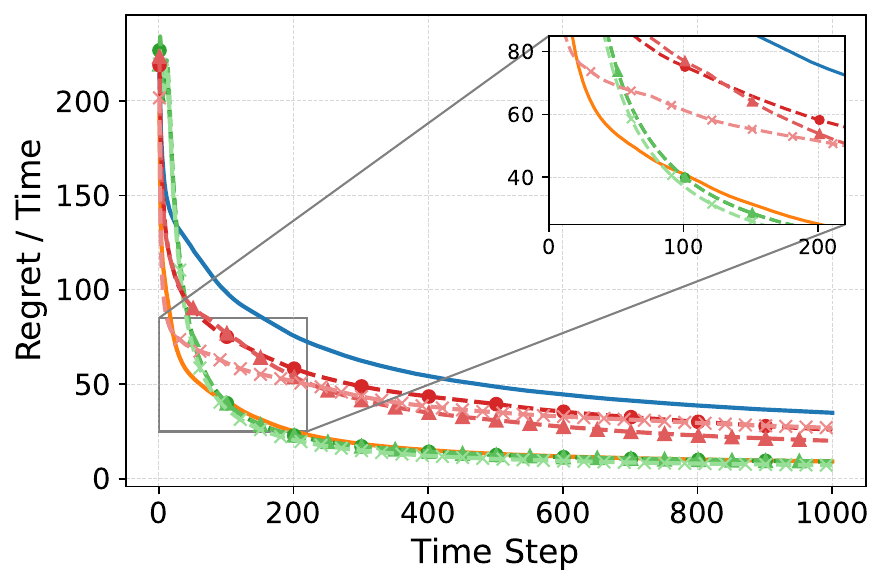}
  \centering
  \small (c) Rich action space, $(t, \tfrac{R(t)}{t})$
\end{minipage}
\hspace{0.05\linewidth}
\begin{minipage}[c]{0.45\linewidth}
  \includegraphics[width=\linewidth]{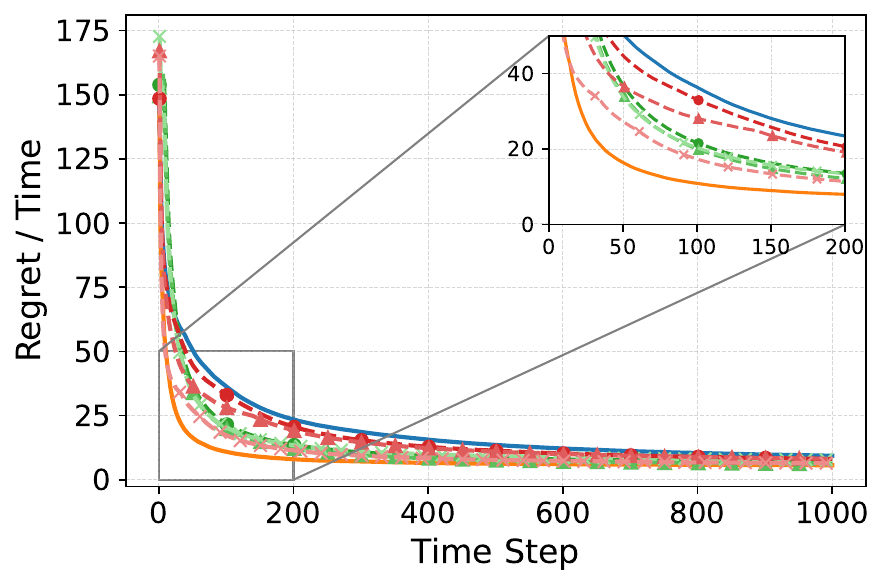}
  \centering
  \small (d) Large context-reward gap, $(t, \tfrac{R(t)}{t})$
\end{minipage}

\caption{
Performance across challenging scenarios. (a) shows the default configuration, serving as a baseline for comparison with the following three cases. (b) introduces a large contextual outlier, where Net-SGD-UCB remains robust due to adaptive variance tracking. (c) increases the number of actions, under which Net-SGD-UCB performs best by efficiently scaling with arm space. (d) features a large reward gap, where Shared LinUCB achieves the lowest regret by rapidly distinguishing optimal arms through aggressive generalization. In contrast, NetLinUCB performs best under the small reward gap in (a), benefiting from fine-grained estimation via network-based inference.
}
\label{fig:exp3to5}
\end{figure}

Table~\ref{tab:radius_decrease_comparison} quantifies the impact of network structure on uncertainty reduction. NetLinUCB achieves up to a 95\% decrease in confidence radius under fully connected settings, highlighting the benefits of structured parameter aggregation. In contrast, Net-SGD-UCB maintains consistent reductions across all connectivity levels, reflecting its robustness through adaptive variance tracking. This distinction underscores the complementary strengths of the two methods: NetLinUCB improves significantly with denser graphs, while Net-SGD-UCB performs reliably under both sparse and dense settings.

\begin{table}[hbt]
\centering
\begin{tabular}{c|c|c|c}
\textbf{Method} & \textbf{3×4} & \textbf{6×2} & \textbf{12} \\
\
Net-SGD-UCB & 21.72\% & 22.40\% & 22.66\% \\
NetLinUCB   & 39.97\% & 84.75\% & 94.82\% \\
\end{tabular}
\caption{Confidence interval reduction ratio compared to Disjoint LinUCB, under different network connectivity.}
\label{tab:radius_decrease_comparison}
\end{table}

We further analyze the evolution of confidence intervals produced by different algorithms. As shown in Figure~\ref{fig:city_estimation_band}, all methods gradually converge toward the true revenue curve. Disjoint LinUCB exhibits slower convergence due to its independent learning per node. Shared LinUCB and NetLinUCB demonstrate faster reduction in uncertainty by exploiting cross-node parameter sharing. Notably, Net-SGD-UCB begins with wider intervals due to its adaptive gradient-based updates, but achieves consistent narrowing over time, thanks to its EMA-based variance control. This highlights the trade-off between stability and adaptivity in decentralized high-dimensional environments.

\begin{figure}[ht]
  \centering
  \includegraphics[width=0.8\linewidth]{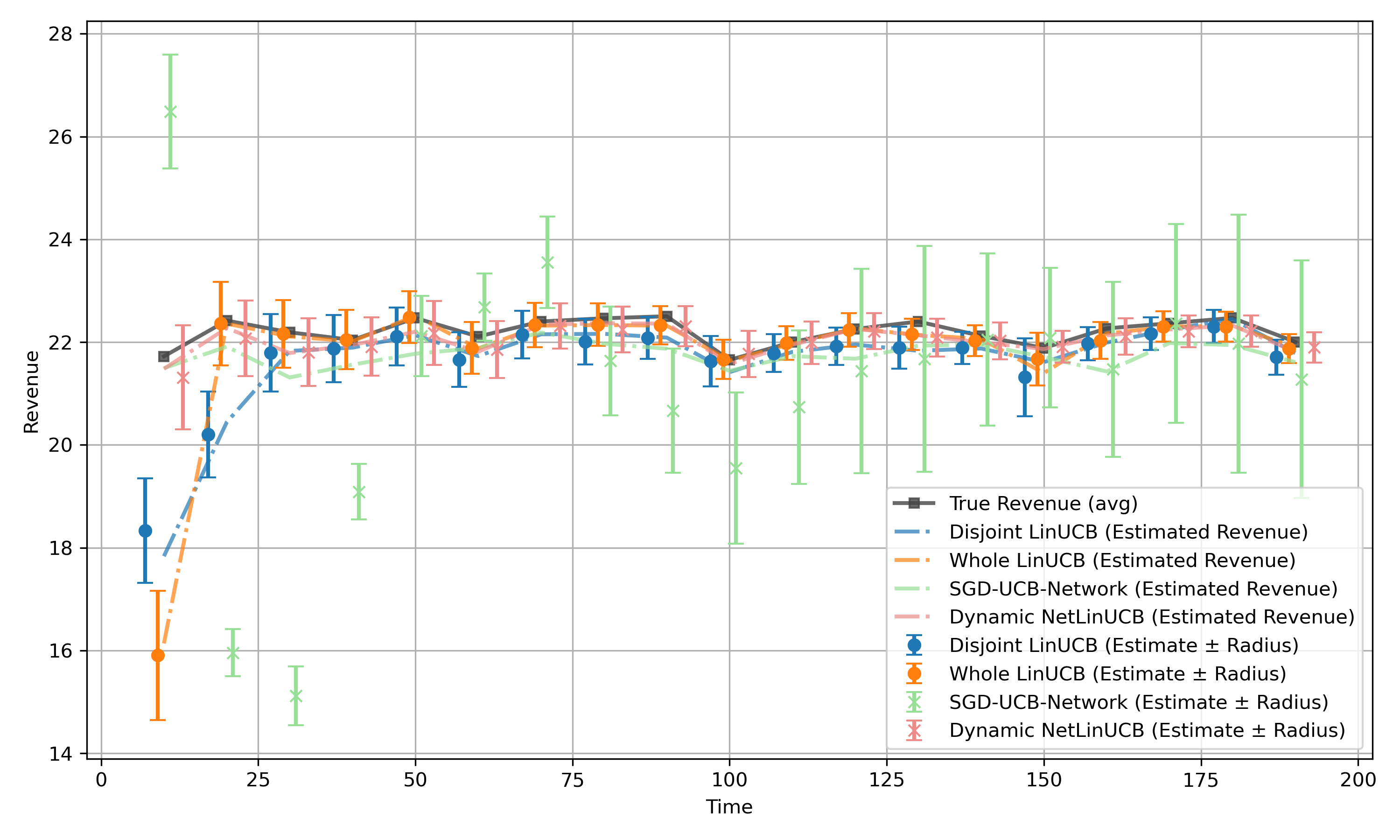}
  \caption{True revenue curve and estimated confidence intervals across all cities under different algorithms.}
  \label{fig:city_estimation_band}
\end{figure}

\end{APPENDICES}

\end{document}